\documentclass[sn-mathphys,natbib]{sn-jnl}% Math and Physical Sciences Reference Style
\usepackage{graphicx}%
\usepackage{multirow}%
\usepackage{subcaption}
\usepackage{amsmath,amssymb,amsfonts}%
\usepackage{amsthm}%
\usepackage{mathrsfs}%
\usepackage[title]{appendix}%
\usepackage{xcolor}%
\usepackage{textcomp}%
\usepackage{manyfoot}%
\usepackage{booktabs}%
\usepackage{algorithm}%
\usepackage{algorithmicx}%
\usepackage{algpseudocode}%
\usepackage{listings}%
\usepackage{natbib}

\theoremstyle{thmstyleone}%
\newtheorem{theorem}{Theorem}%  meant for continuous numbers
\theoremstyle{thmstyletwo}%

\theoremstyle{thmstylethree}%
\newtheorem{definition}{Definition}%

\begin{document}

\title[Article Title]{RandomNet: Clustering Time Series Using Untrained Deep Neural Networks}


\author[1]{\fnm{Xiaosheng} \sur{Li}}\email{xli22@gmu.edu}
\equalcont{These authors contributed equally to this work.}

\author*[1]{\fnm{Wenjie} \sur{Xi}}\email{wxi@gmu.edu}
\equalcont{These authors contributed equally to this work.}

\author[1]{\fnm{Jessica} \sur{Lin}}\email{jessica@gmu.edu}

\affil*[1]{\orgdiv{Department of Computer Science}, \orgname{George Mason University}, \orgaddress{\street{4400 University Dr}, \city{Fairfax}, \postcode{22030}, \state{Virginia}, \country{USA}}}


\abstract{Neural networks are widely used in machine learning and data mining. Typically, these networks need to be trained, implying the adjustment of weights (parameters) within the network based on the input data. In this work, we propose a novel approach, RandomNet, that employs untrained deep neural networks to cluster time series. RandomNet uses different sets of random weights to extract diverse representations of time series and then ensembles the clustering relationships derived from these different representations to build the final clustering results. By extracting diverse representations, our model can effectively handle time series with different characteristics. Since all parameters are randomly generated, no training is required during the process. We provide a theoretical analysis of the effectiveness of the method. To validate its performance, we conduct extensive experiments on all of the 128 datasets in the well-known UCR time series archive and perform statistical analysis of the results. These datasets have different sizes, sequence lengths, and they are from diverse fields. The experimental results show that the proposed method is competitive compared with existing state-of-the-art methods.}




\maketitle

\section{Introduction}\label{sec:intro}
Neural networks serve as fundamental learning models across disciplines such as machine learning, data mining, and artificial intelligence. Typically, these networks go through a training phase during which their parameters are tuned according to specific learning rules and the data provided. A popular training paradigm involves backpropagation for optimizing an objective function. Once trained, these networks can be deployed for a variety of tasks, including classification, clustering, and regression.

A time series is a real-valued ordered sequence. The task of time series clustering assigns time series instances into homogeneous groups. It is one of the most important and challenging tasks in time series data mining and has been applied in various fields such as finance \citep{finance}, biology \citep{biology,biology2}, climate \citep{climate}, medicine \citep{medicine} and so on. In this work, we consider the partitional clustering problem, wherein the given time series instances are grouped into pairwise-disjoint clusters.

Existing time series clustering methods achieve good performance~\citep{kshape, DBA, spf}, but since they form clusters based on a single focus, such as shape or point-to-point distance, they are suboptimal for some specific data types. Here, we introduce a novel method named RandomNet for time series clustering using untrained deep neural networks. Different from conventional training methods that adjust network weights (parameters) using backpropagation, RandomNet utilizes different sets of random parameters to extract various representations of the data. By extracting diverse representations, it can effectively handle time series with different characteristics. These representations are clustered; the results from the clusters are then selected and ensembled to produce the final clustering. This approach ensures that data only needs to pass through the networks once to obtain the final result, obviating the need for backpropagation. Therefore, the time complexity of RandomNet is linear in the number of instances in the dataset, providing a more efficient solution for time series clustering tasks.

Given a neural network, the various sets of parameters in the network can be thought of as performing different types of feature extraction on the input data. As a result, these varied parameters can generate diverse data representations. Some of these representations may be relevant to a clustering task, producing meaningful clusterings, while others may be less useful or entirely irrelevant, leading to less accurate or meaningless clustering. This concept forms the basis of RandomNet: by combining clustering results derived from all these diverse representations, the meaningful and latent group structure within the data can be discovered. This is because the noise introduced by irrelevant representations tends to cancel each other out during the ensemble process whereas the connections provided by relevant representations are strengthened. Therefore, efficient and reliable clustering can be achieved despite the randomness of the network parameters.

To demonstrate the effectiveness of RandomNet, we provide theoretical analysis. The analysis shows that RandomNet has the ability to effectively identify the latent group structure in the dataset as long as the ensemble size is large enough. Moreover, the analysis also provides a lower bound for the ensemble size. Notably, this lower bound is independent of the number of instances or the length of the time series in the dataset, given that the data in the dataset are generated from the same mechanism. This provides the ability to use a fixed, large ensemble size to achieve satisfactory results, offering a practical approach to time series clustering that does not need adjustment for different dataset sizes or time series lengths.

We conduct extensive experiments on all 128 datasets in the well-known UCR time series archive \citep{UCRArchive} and perform statistical analysis on the results. These datasets have different sizes, sequence lengths, and characteristics. The results show that RandomNet has the top performance in the Rand Index compared with other state-of-the-art methods and achieves superior performance across all data types evaluated.

The main contributions of the paper are summarized as follows: 
\begin{itemize}
    \item We propose RandomNet, a novel method for time series clustering using untrained neural networks with random weights. There is no training or backpropagation in the method. 
    \item We demonstrate the effectiveness of the proposed method both empirically and theoretically. We conduct extensive experiments on 128 datasets to evaluate the proposed method and provide statistical analysis of the comparison results to show the superiority of our method over the state-of-the-art methods. 
    \item We demonstrate the efficiency of the proposed method through the experimental evaluation on data of varying sizes and time lengths. The results of linear curve-fitting on the running time indicate that the method has linear time complexity.
\end{itemize}

\section{Background and Related Work}\label{sec:background}
\subsection{Definitions and notations}
\begin{definition}
A time series $T=[t_1, t_2, \ldots, t_m]$ is an ordered sequence of real-value data points, where $m$ is the length of the time series.
\end{definition}

\begin{definition}
Given a set of time series $\{T_i\}_{i=1}^n$ and the number of clusters $k$, the objective of time series clustering is to assign each time series instance $T_i$ a group label $c_j$, where $j \in \{1, \ldots, k\}$. $n$ is the number of instances in the dataset. We would like the instances in the same group to be similar to each other and dissimilar to the instances in other groups.
\end{definition}

\subsection{Related work}\label{sec:related}
There has been much work on time series clustering, and we categorize them into four groups: raw-data-based methods, feature-based methods, deep-learning-based methods, and others.

\textbf{Raw-data-based methods.}
The raw-data-based methods directly apply classic clustering algorithms such as k-means \citep{kmeans} on raw time series. The standard k-means algorithm adopts Euclidean distance to measure the dissimilarity of the instances and often cannot handle the scale-variance, phase-shifting, distortion, and noise in the time series data. To cope with these challenges, dozens of distance measures for time series data have been proposed.

Dynamic Time Warping (DTW) \citep{DTW} is one of the most popular distance measures that can find the optimal alignment between two sequences. It is used in Dynamic time warping Barycenter Averaging (DBA) \citep{DBA} which proposes an iterative procedure to refine the centroid in order to minimize the squared DTW distances from the centroids to other time series instances. Similarly, K-Spectral Centroid (KSC) \citep{KSC} proposes a distance measure that finds the optimal alignment and scaling for matching two time series. The centroids are computed, based on matrix decomposition, to minimize the distances between the centroids and the instances under this distance measure.
Another approach, k-shape \citep{kshape} proposes a shape-based distance measure based on the cross-correlation of two time series. The distance measure shifts the two time series to find the optimal matching. Each centroid is obtained from the optimization of the squared normalized cross-correlation from the centroid to the instances in the cluster.

\textbf{Feature-based methods.}
Feature-based methods transform the time series into flat, unordered features, and then apply classic clustering algorithms to the transformed data.

Zakaria et al. \citep{ushapelet} propose to calculate the distances from a set of short sequences to the time series instances in the dataset and use the distance values as new features for the respective instances. This set of short sequences, called U-shapelets, is found by enumerating all the subsequences in the data to best separate the instances. K-means are then applied to the new features for clustering. In the work by Zhang et al. \citep{uslm}, instead of enumerating the subsequences, the shapelets are learned by optimizing an objective function with gradient descent.

A recent work \citep{SPIRAL} proposes Similarity PreservIng RepresentAtion Learning (SPIRAL) to sample pairs of time series to calculate their DTW distances and build a partially-observed similarity matrix.  The matrix is an approximation for the pair-wise DTW distances matrix in the dataset. The new features are generated by solving a symmetric matrix factorization problem such that the inner product of the new feature matrix can approximate the partially-observed similarity matrix.

\textbf{Deep-learning-based methods.}
Many methods in this category adopt the autoencoder architecture for clustering. In autoencoder, the low-dimension hidden layer output is used as features for clustering. Among these, Improved Deep Embedded Clustering (IDEC) \citep{IDEC} improves autoencoder by adding an extra layer to the model. It not only employs a reconstruction loss but also optimizes a clustering loss specifically designed to preserve the local structure of the data. This dual loss strategy can capture the global structure and local differences, thereby improving the clustering process to better learn the inherent characteristics of the data.

Deep Temporal Clustering (DTC) \citep{DTC} specifically addresses time series clustering by using Mean Square Error (MSE) to measure the reconstruction loss, and Kullback-Leibler (KL) divergence to measure clustering loss. Similarly, Deep Temporal Clustering Representation (DTCR) \citep{DTCR} adopts MSE for the reconstruction loss, while it uses a k-means objective function to measure the clustering loss. DTCR also employs a fake-sample generation strategy to augment the learning process. Clustering Representation Learning on Incomplete time-series data (CRLI) \citep{CRLI} further studies the problem of clustering time series with missing values. It jointly optimizes the imputation and clustering process, aiming to impute more discriminative values for clustering and to make the learned representations possess a good clustering property.

In the broader neural network literature, there is a class of methods that also use random weights known as the Extreme Learning Machine (ELM) \citep{ELM1, ELM2, ELM3}, which uses a single-layer feed-forward network to map inputs into a new feature space. The hidden layer weights are set randomly but the output weights are trained. The idea is to find a mapping space where instances of different classes can be separated well.

In the domain of time series classification, ROCKET \citep{ROCKET}, MiniRocket \citep{MINIROCKET} and MultiRocket \citep{multirocket} adopt strategies involving the use of random weights to generate features for classification. They use multiple single-layer convolution kernels instead of a deep network architecture.

Beyond the neural network and clustering fields, several works also adopt randomized features or feature maps \citep{RF1, RF2, RF3}. However, it is worth noting that all these methods diverge from our proposed approach in their network structures. Moreover, none of these methods incorporates ensemble learning, which forms the core of our approach. To the best of our knowledge, we are the first to propose using a network with random weights in time series clustering.

\textbf{Other methods.}
In our previous work \citep{spf}, we present a Symbolic Pattern Forest (SPF) algorithm for time series clustering, which adopts Symbolic Aggregate approXimation (SAX) \citep{SAX} to transform time series subsequences into symbolic patterns. Through iterative selections of random symbolic patterns to divide the dataset into two distinct branches based on the presence or absence of the pattern, a symbolic pattern tree is constructed. Repeating this process forms a symbolic pattern forest, the ensemble of which produces the final clustering result.

\section{The Proposed Method}
\label{sec:method}

\begin{figure*}[]
\centering
\includegraphics[width=1\textwidth, height=0.45\textheight]{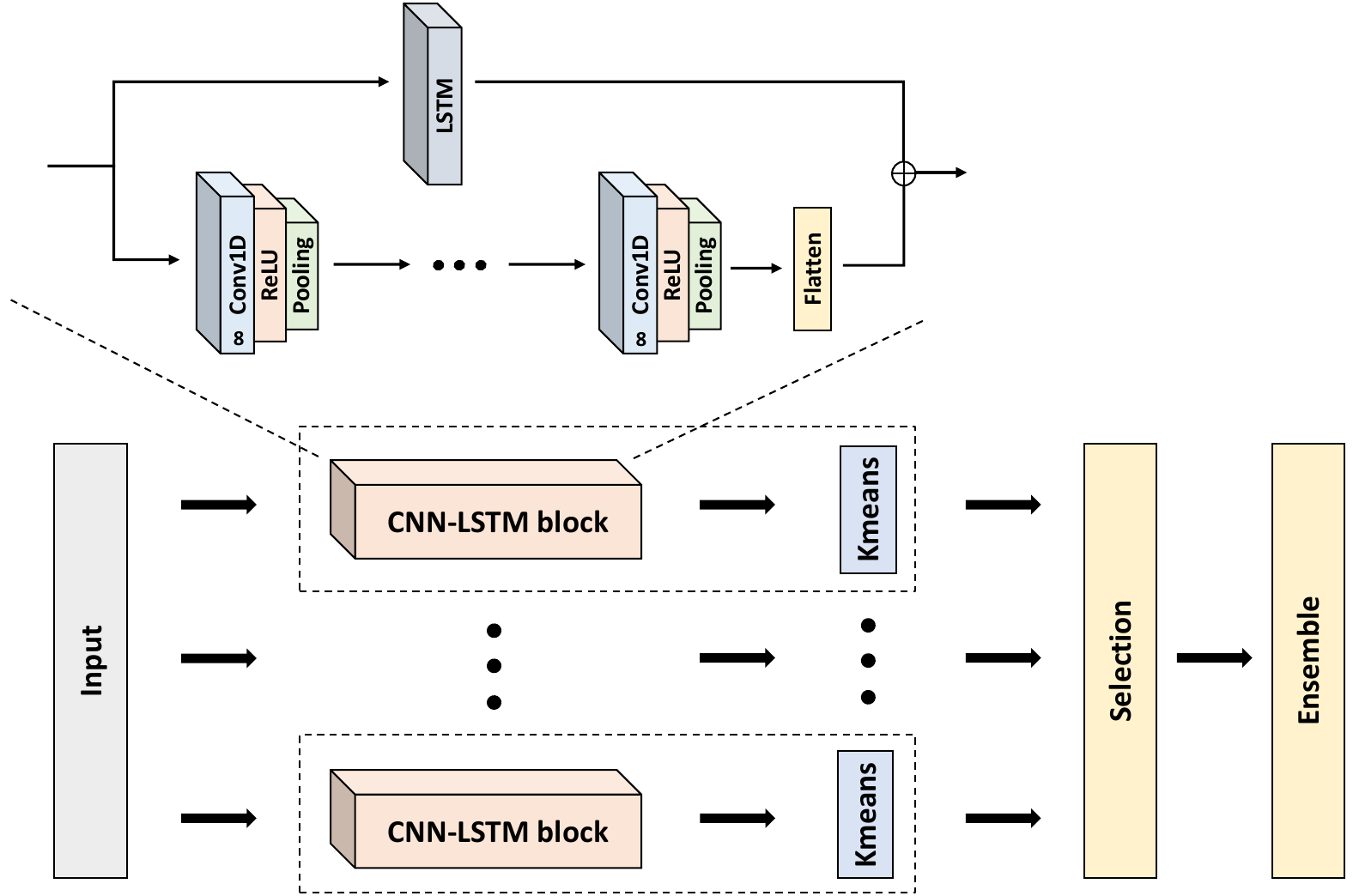}
\caption{The overall structure of RandomNet.}
\label{fig:architecture}
\end{figure*}

\subsection{Architecture and algorithm}

Figure~\ref{fig:architecture} shows the architecture of RandomNet. The method is structured with $B$ branches, each containing a CNN-LSTM block, designed to capture both spatial and temporal dependencies of time series, followed by k-means clustering. Each CNN-LSTM block contains multiple groups of CNN networks and an LSTM network, and each group of CNN network consists of a one-dimensional convolutional layer, a Rectified Linear Units (ReLU) layer, and a pooling layer. The output of the CNN networks is flattened. In our experiments, we set the number of groups of the CNN network equal to $\log_2{m}$, where $m$ represents the length of the time series. We fix the number of filters of the 1D convolution to 8, the filter size to 3, and the pooling size to 2. We set the number of LSTM units to 8. The weights used within the network are randomly chosen from $\{-1, 0, 1\}$. We opt for this finite parameter set over a continuous interval (e.g., $[-1, 1]$) for the purpose of simplifying the parameter space.

Each branch produces its own clustering, however, some clusterings might be skewed or deviant due to the inherent randomness of the weights. To alleviate this problem, we propose a selection mechanism to remove any clusterings that contain clusters that are either too small or too large.

Concretely, the method sets a lower bound $lr$ and an upper bound $ur$ for the cluster size. The number of instances that violate the bounds in each clustering is counted as \textit{violation}. For example, suppose a clustering contains two clusters with sizes 40 and 52, respectively. If the lower bound is 5 and the upper bound is 50, then the number of \textit{violations} for this clustering is $52-50=2$. The clusterings are sorted according to the number of \textit{violations} and the method selects the top $S$ clusterings for the ensemble. Here, $S=\max(zv, sr \times B)$, where $zv$ is the number of clusterings with zero violation values, $sr$ is a selection rate, and $B$ is the number of branches in the method.

Finally, we ensemble the results to form the final clustering. While the diversity of clustering results from a large number of different branches helps reveal various intrinsic patterns in the data, it introduces the challenge of combining these different results into a cohesive unified clustering. To address this challenge, we adopt the Hybrid Bipartite Graph Formulation (HBGF) \citep{HBGF} to perform clustering ensemble.
% due to its efficiency and effectiveness in reconciling the variances across multiple clusterings
This technique builds a bipartite graph for the clusterings in the ensemble, where the instances and clusters become the vertices. If an instance belongs to a cluster, then there is an edge connecting the two respective vertices in the graph. Partitioning the graph gives a consensus clustering for the ensemble. HBGF has two main advantages. First, it can extract consensus from differences, identifying and strengthening the repeated patterns of grouping across the clustering set. Second, it has linear time complexity, which ensures the scalability of our model for large datasets. In our implementation, we use Metis \citep{metis} library to partition the graph.

Algorithm 1 gives the pseudo-code of RandomNet. Given a time series dataset $D=\{T_i\}_{i=1}^n$, a branch number $B$, a cluster number $k$, bounds $lr$ and $ur$, and a selection rate $sr$, the algorithm outputs a clustering assignment $C$ for the input time series.

\begin{algorithm}[t]
\caption{RandomNet}

\begin{algorithmic}[1]
\Require
    $D$: time series dataset; 
    $B$: branch number; 
    $k$: number of clusters; 
    $lr, ur$: lower and upper bounds; 
    $sr$: selection rate
\Ensure 
    $C$: clustering assignment of $D$
\Function{RandomNet}{$D$, $B$, $k$, $lr$, $ur$, $sr$}
    \State $ClusteringSet \gets$ empty
    \For {$i = 1$ to $B$}
        \State $Randomize(CNN\_LSTM\_blocks)$
        \State $Features \gets CNN\_LSTM\_blocks(D)$
        \State $Clustering \gets Kmeans(Features)$
        \State $ClusteringSet.add(Clustering)$
    \EndFor
    \State $SelectedSet \gets Selection(ClusteringSet, sr, lr, ur, B)$  
         \hfill // Algorithm 2
    \State $C \gets Ensemble(SelectedSet)$ \hfill // Algorithm 3
    \State \Return $C$
\EndFunction
\end{algorithmic}

\end{algorithm}

\begin{algorithm}[]
\caption{Selection Mechanism}
\begin{algorithmic}[1]
\Require
    $ClusteringSet$: a set of clusterings; 
    $sr$: selection rate;
    $lr, ur$: lower and upper bounds; 
    $B$: branch number
\Ensure 
    $SelectedSet$: a subset of $ClusteringSet$ selected based on criteria

\Function{Selection}{$ClusteringSet, sr, lr, ur, B$}
    \State $SelectedSet \gets$ empty
    \State $ViolationsList \gets$ empty list of $Size(ClusteringSet)$
    \For{$i = 1$ to $Size(ClusteringSet)$}
        \State $Clustering \gets ClusteringSet[i]$
        \State $Violations \gets 0$
        \For{$Cluster$ in $Clustering$}
            \If{$Size(Cluster) > ur$}
                \State $Violations \gets Violations + (Size(Cluster) - ur)$
            \ElsIf{$Size(Cluster) < lr$}
                \State $Violations \gets Violations + (lr - Size(Cluster))$
            \EndIf
        \EndFor
        \State $ViolationsList[i] \gets Violations$
    \EndFor
    \State $SortedIndices \gets$ Sort indices of $ViolationsList$ in ascending order of violations
    \State $EffectiveSize \gets$ Max($sr \times B$, number of $Clustering$ in $ClusteringSet$ with $0$ violations)
    \For{$i = 1$ to $EffectiveSize$}
        \State $SelectedSet$.add($ClusteringSet[SortedIndices[i]]$)
    \EndFor
    \State \Return $SelectedSet$
\EndFunction
\end{algorithmic}
\end{algorithm}

\begin{algorithm}
\caption{Ensemble}
\begin{algorithmic}[1]
\Require
    $SelectedSet$: a selected set of $ClusteringSet$
\Ensure 
    $C$: final clustering assignment

\Function{Ensemble}{$SelectedSet$}
    \State $Graph \gets$ Construct bipartite graph from $SelectedSet$
    \State $C \gets$ Partition $Graph$ using Hybrid Bipartite Graph Formulation 
    \State \Return $C$
\EndFunction
\end{algorithmic}
\end{algorithm}

In Algorithm 1 Line 4, the parameters in the CNN-LSTM blocks are randomly set from $\{-1, 0, 1\}$ as previously noted. The data passes the CNN-LSTM blocks to generate features for each time series in Line 5. Line 6 applies k-means on the features to produce a clustering assignment. Line 7 adds the clustering to the ensemble set. In Line 9, the selection mechanism (Algorithm 2) introduced above is performed on the ensemble set with the user-provided selection rate and bounds. Finally in Line 10, the ensemble function (Algorithm 3) ensembles the clusterings in $SelectedSet$ and gives the clustering $C$ as the output of the algorithm. 

\subsection{Effectiveness of RandomNet}
Given the network architecture, its parameters (weights) represent a form of feature extraction from the data and thus produce a kind of representation. With multiple random parameters, we can have multiple representations.

Some representations are relevant to the clustering task. The instances that are similar to each other are more likely to be put in the same cluster under these relevant representations. Other representations are irrelevant to the clustering task. Under these representations, two similar instances may not be assigned in the same cluster.

The intuition is that, in the ensemble, the effect of irrelevant representations can cancel each other out, and the effect of relevant representations can dominate the ensemble. Inspired by \citep{spf} which is described in the previous section, we provide effectiveness analysis for RandomNet.

We assume the data contains $k$ distinct clusters which correspond to $k$ different classes. We have the following theorem:

\begin{theorem}
    Assume two instances, $T_1$ and $T_2$, are from the same class. If they reside within the same cluster under some relevant representations, then RandomNet assigns these two instances to the same cluster in the final output.
\end{theorem}
\begin{proof}
    Let $\gamma$ denote the percentage of relevant representations among all the representations. In each CNN-LSTM block, if the representation is relevant, we have $P(C(T_1)=C(T_2))=1$,
    where $P(\cdot)$ stands for the probability and $C(\cdot)$ denotes the clustering assignment. If the representation is irrelevant, the instances are assigned to any of the $k$ clusters randomly. Hence, we can deduce: $P(C(T_1)=C(T_2))=1/k$, and $P(C(T_1) \neq C(T_2))=(k-1)/k$.
    Considering the above, overall we can derive that $P(C(T_1)=C(T_2))= \gamma \times 1 + (1- \gamma) \times 1/k$ and $P(C(T_1) \neq C(T_2))= (1- \gamma) \times (k-1)/k$. It is clear that $P(C(T_1)=C(T_2))> P(C(T_1) \neq C(T_2))$.
    Since each block is independent of the others, according to the law of large numbers \citep{lln} which states that if we repeat an experiment independently a large number of times, the average of the results obtained from those experiments will converge to the expected value, we have:
    \begin{equation}
    \label{eqn_count}
    Count(C(T_1)=C(T_2))> Count(C(T_1) \neq C(T_2))
    \end{equation}
    where we consider a sufficiently large ensemble size and $Count(\cdot)$ is the count of occurrences. Consequently, in the ensemble result, instances $T_1$ and $T_2$ belong to the same cluster.
\end{proof}

The above analysis assumes we have a large ensemble size, and the following theorem provides a lower bound for the ensemble size. Here, for simplicity, we set $k=2$.

\begin{theorem}
    Assume the ensemble size to be $b$. Then, the lower bound of $b$ needed to provide a good clustering is given by $-2 \ln \alpha / \gamma^2$, where $\gamma$ represents the percentage of relevant representations and 1-$\alpha$ is the confidence level.
\end{theorem}
\begin{proof}
    Let $Y$ be a random variable indicating the number of cases where $C(T_1)=C(T_2)$. The random variable $Y$ follows a binomial distribution:
    \begin{equation}
    P(Y=s) = \binom{b}{s}p^s(1-p)^{b-s}
    \end{equation}
    where $p=P(C(T_1)=C(T_2))$. Equation (\ref{eqn_count}) needs to hold with high probability, leading to the following inequality:
    \begin{equation}
    P(Y \le s) = \sum_{i=0}^s \binom{b}{i}p^i(1-p)^{b-i} \le \alpha
    \end{equation}
    where $s=b/2$, $1-\alpha$ is the confidence level. By applying Hoeffding's inequality \citep{hoeffding}, we have $P(E[\bar{Y}]-\bar{Y} \ge t) \le e^{-2bt^2}$,
    where $t \ge 0$. Considering $E[\bar{Y}]=p$, we have:
    \begin{align}
    P(E[\bar{Y}]-\bar{Y} \ge t) &= P(bE[\bar{Y}]-b\bar{Y} \ge bt)\\
     %&= P(qp-X \ge qt)\\
     &= P(Y \le bp-bt) \le e^{-2bt^2}
    \end{align}
    Let $s=bp-bt$, then $t=(bp-s)/b$, so we get $P(Y \le s) \le e^{-2(bp-s)^2/b} \le \alpha$.
    With $s=b/2$, $p=\gamma \times 1 + (1- \gamma) \times 1/2$, we solve the above inequality and derive $b \ge -2 \ln \alpha / \gamma^2$.
\end{proof}

Here is a concrete example for the bound: suppose we have a confidence level of 99\% and we estimate that 30\% of the representations are relevant. In this case, $\alpha=0.01$ and $\gamma=0.3$, yielding a $b$ value of at least 102.33. From the theorem, one observes that the lower bound is independent of the number of instances in the dataset. This implies that we can maintain a sufficiently large fixed ensemble size to handle inputs of varying sizes, provided the data generation mechanism remains constant. We verify this in the experimental section by using a fixed ensemble size chosen through experiments and varying the number of time series instances that are generated from the same mechanism.

\section{Experimental Evaluation}
\label{sec:experiment}

\subsection{Experimental setup}

To evaluate the effectiveness of RandomNet, we run the algorithm on all 128 datasets from the well-known UCR time series archive \citep{UCRArchive}. These datasets come from different disciplines with various characteristics. Each dataset in the archive is split into a training set and a testing set. We fuse the two sets and utilize the entire dataset in the experiment. Some of these datasets contain varying-length time series. To ensure that all time series in a dataset have the same length, we append zeros at the end of the shorter series.

For benchmarking purposes, we run kDBA \citep{DBA}, KSC \citep{KSC}, k-shape \citep{kshape}, SPIRAL \citep{SPIRAL}, and SPF~\citep{spf} on the same datasets. These methods are used as representatives of the state-of-the-art for time series clustering. Additionally, we incorporated deep-learning-based methods, Improved Deep Embedding Clustering (IDEC) \citep{IDEC} and DTC \citep{DTC}, for comparison. While DTC is specifically designed for time series data, as discussed in Section \ref{sec:related}, IDEC is a general clustering method. We also compare our method with ROCKET \citep{ROCKET} and its variants, MiniRocket \citep{MINIROCKET} and MultiRocket \citep{multirocket}, since we are interested in how other models that also used random parameters compare to ours. As they are all specifically designed for time series classification, we adapt them to our use case by removing the classifier component and replacing it with k-means. All references to them will pertain to this adapted version. We do not include DTCR \citep{DTCR} in the comparison, as we are unable to reproduce the results reported in its paper, despite using the code provided by its authors\footnote{https://github.com/qianlima-lab/DTCR}. This issue has been similarly reported by others on the GitHub issue webpage for the project\footnote{https://github.com/qianlima-lab/DTCR/issues/8}. We do not include CRLI \citep{CRLI} since it is specially designed for incomplete time series data which is outside the scope of our study. Table \ref{tab:comparison} provides a concise comparison of our method and various baselines we used in experiments, outlining their applicable data types, method types, main focuses, and time complexity in terms of the number of instances ($n$) and the length of time series ($l$). Note that due to the complexity involved in training deep learning models, we have not included the time complexity for the two deep learning methods, DTC and IDEC, which require network training. For a more detailed description of each method, please refer to Section \ref{sec:related}. We provide the experimental evaluation for the time complexity of our model in Section \ref{sec:time}.

\begin{table}[]
\centering
\scalebox{0.9}{\begin{tabular}{|l|l|l|p{4cm}|p{3cm}|}
\hline
\textbf{Method} & \textbf{Data Type} & \textbf{Method Type} & \textbf{Main Focus} & \textbf{Time Complexity}\\
\hline
K-means & Varied & Raw-data-based & Centroid-based clustering for partitioning. & $O(n\cdot l)$ \\
\hline
KSC & Time Series & Raw-data-based & Time series clustering by patterns. & $O(max(n\cdot l^2, l^3))$\\
\hline
k-shape & Time Series & Raw-data-based & Shape-based clustering for time series. & $O(max(n\cdot l\cdot log(l), n\cdot l^2, l^3))$\\
\hline
SPF & Time Series & Others & Linear complexity clustering using symbolic patterns. & $O(n\cdot l)$ \\
\hline
SPIRAL & Time Series & Feature-based & Representation learning for time series clustering. & $O(n\cdot log(n)\cdot l^2)$ \\
\hline
kDBA & Time Series & Raw-data-based & Averaging sequences with DTW for clustering. & $O(n\cdot l^2)$\\
\hline
IDEC & Varied & Deep-learning-based & Deep clustering with local structure preservation. & N/A\\
\hline
DTC & Time Series & Deep-learning-based & Integrates dimensionality reduction with time series clustering. & N/A\\
\hline
MiniRocket & Time Series & Deep-learning-based & Fast deep neural network with random weights for time series classification. & $O(n\cdot l)$\\
\hline
RandomNet & Time Series & Deep-learning-based & Untrained deep learning architecture that generates diverse representations. & $O(n\cdot l)$\\
\hline
\end{tabular}}
\caption{Comparison of baselines and our method.}
\label{tab:comparison}
\end{table}

The source code of kDBA, KSC, and k-shape are obtained from the authors of k-shape. The source code of SPIRAL\footnote{https://github.com/cecilialeiqi/SPIRAL}, SPF\footnote{https://github.com/xiaoshengli/SPF}, IDEC\footnote{https://github.com/XifengGuo/IDEC}, DTC\footnote{https://github.com/FlorentF9/DeepTemporalClustering},
ROCKET\footnote{https://github.com/angus924/rocket},
MiniRocket\footnote{https://github.com/angus924/minirocket} and
MultiRocket\footnote{https://github.com/ChangWeiTan/MultiRocket}is available online. The number of clusters $k$ is set to equal the number of classes in the datasets, and we follow the default parameter settings in the source code. 

For the default hyperparameters of RandomNet, the number of branches $B$ is set to 800, the selection rate $sr$ is set to 0.1, the lower bound $lr$ is set to $0.3 \times acs$, and the upper bound $ur$ is set to $1.5 \times acs$, where $acs$ refers to the average cluster size, which is computed as $round(number\_of\_instances/k)$. Detailed hyperparameter selection experiments are elaborated upon in the subsequent section. 

We implement RandomNet using Python and TensorFlow 2.1 and use the k-means implementation in the Scikit-learn package \citep{sklearn} with default settings. The experiments are run on a node of a batch-processing cluster. The node uses a 2.6 GHz CPU and 64 GB RAM. Given that the process does not involve neural network training, there is no necessity for GPU utilization. %The source code will be made publicly available upon publication. 
The source code of RandomNet is available at:  \url{https://github.com/Jackxiini/RandomNet}.

All the datasets in the experiments have labels that can be used as ground truth. We use Rand Index \citep{randindex} to measure the clustering accuracy of the methods under comparison. %Rand Index is defined as $Rand \; Index = (TP + TN)/(TP + TN + FP + FN)$,
%\begin{equation}
%Rand \; Index = \frac{TP + TN}{TP + TN + FP + FN}
%\end{equation}
%where $TP$ (True Positive) is the number of instances from the same class and in the same cluster. $TN$ (True Negative) is the number of instances from different classes assigned in different clusters. $FP$ (False Positive) is the number of instances from different classes in the same cluster. $FN$ (False Negative) is the number of instances from the same class assigned in different clusters. 
The range of the Rand Index falls in $[0, 1]$ where a large value indicates that the clustering matches the actual class relationship well.

\begin{figure*}[]
\centering
\includegraphics[width=1\textwidth]{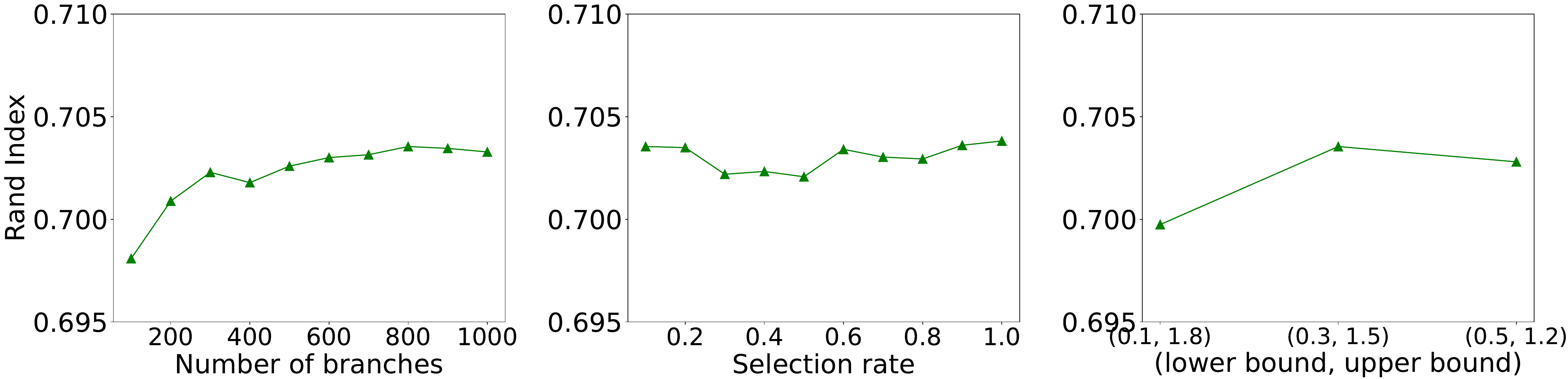}
\caption{Left: Rand Index by the varying number of branches, ranging from 100 to 1000. Middle: Rand Index by the varying selection rate, ranging from 0.1 to 1. Right: Rand Index by three sets of the lower bound and upper bound, representing wide interval, intermediate interval, and narrow interval.}
\label{fig:hyperp}
\end{figure*}

\subsection{Hyperparameter analysis}
To fine-tune and investigate the influence of hyperparameters on model performance, we conduct a series of experiments. We select 20 datasets from the UCR time series archive \citep{UCRArchive} and run each experiment 10 times and take the average Rand Index as the result. 

\textbf{Number of branches.}
The number of branches $B$ plays a pivotal role in our model, affecting both the quality of the clustering and the computational efficiency. We test $B$ values ranging from 100 to 1000, in increments of 100, and keep all other settings default.

The left plot of Fig. \ref{fig:hyperp} shows that the average Rand Index improves with an increase in the number of branches until $B=800$. %Beyond this point, however, we observe no significant further enhancements in the Rand Index, suggesting $B=800$ as the optimal choice. 
Increasing the number of branches beyond 800 only results in a rise in running time, without contributing to better clustering quality. Therefore, we set the default $B$ for all datasets to 800.

\textbf{Selection rate.}
The selection rate $sr$ controls the lower bound of the number of selected clustering. We test $sr$ values ranging from 0.1 to 1, in increments of 0.1, and keep all other settings default.

The middle plot of Fig. \ref{fig:hyperp} shows slight changes in the average Rand index as $sr$ changes. Since a larger $sr$ will increase the running time of the model, we choose $sr=0.1$ as the default value.

\textbf{Lower bound and upper bound.}
The lower bound $lr$ and the upper bound $ur$ are crucial in detecting the number of \textit{violation}, which affects the quality of clustering. We evaluate three pairs of $lr$ and $ur$, $(0.1, 1.8), (0.3, 1.5)$, and $(0.5, 1.2)$, representing wide intervals, intermediate intervals, and narrow intervals, respectively. For simplicity, we present these as multipliers; the actual lower and upper bounds are obtained by multiplying these values with the average cluster size $acs$. Narrower intervals are more restrictive to the size of the clustering and thus will increase the number of \textit{violations}. We keep all other settings as default.

The right plot of Fig. \ref{fig:hyperp} shows the effects of the wide interval, intermediate interval, and narrow interval on the Rand Index. We can observe that appropriate intervals can bring better performance. The intermediate interval has the best Rand Index, whereas the wide interval underperforms the other two due to its lax size constraints which affects its ability to screen \textit{violations}. Therefore, we set $lr$ to $0.3 \times acs$ and $ur$ to $1.5 \times acs$ as default.

\textbf{Selection bias test.} Since the dataset selection for hyperparameter tuning may cause potential selection bias, we conduct an experiment to show how the choice affects the experimental results. We divide 128 datasets into five similar-sized groups, optimizing three hyperparameters separately for each. We then apply the optimized hyperparameters to the remaining datasets in each group. For each hyperparameter, we average the Rand Index for each group, and obtain the average and the standard deviation of the Rand Index of the five groups. As illustrated in Table \ref{tab:selection bias}, the standard deviations, 0.0073, 0.0097 and 0.0078, for the Average Rand Index across the five groups are very small, indicating that the selection of datasets for hyperparameter tuning has a minimal effect on the final results. Therefore, we retain the selection of the previous 20 datasets in subsequent experiments.

\begin{table}[]
\caption{Dataset selection bias test for each hyperparameter}
\begin{tabular}{|l|c|c|c|}
\hline
                   & Num.of branches & Selection rate & \multicolumn{1}{c|}{\begin{tabular}[c]{@{}c@{}}Lower bound and \\ upper bound\end{tabular}} \\ \hline
Average Rand Index & $0.7391\pm0.0073$       & $0.7374\pm0.0097$         & $0.7374\pm0.0078$                                                                                      \\ \hline
\end{tabular}
\label{tab:selection bias}
\end{table}

\subsection{Experimental results}
Since we use 20 of the 128 datasets to select the hyperparameters, for the sake of fairness, we remove them in the following comparison and only show the results for the remaining 108 datasets.

\textbf{Comparison with k-means.}

\begin{figure}[]
\centering
\includegraphics[width=0.85\textwidth]{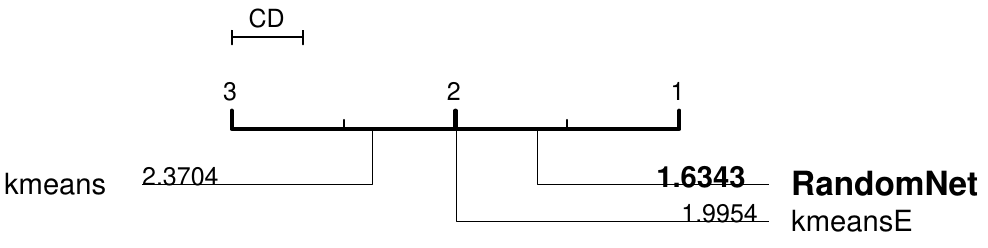}
\caption{Critical difference diagram of the comparison with k-means, kmeansE and RandomNet.}
\label{fig:compareKmeans}
\end{figure}

As RandomNet uses k-means to generate clustering assignments, we are interested in how they compare. We also run k-means 800 times and use HBGF \citep{HBGF} to ensemble the results, which is denoted as kmeansE. 

We run the methods under comparison on the 108 datasets and record the Rand Index. Figure \ref{fig:compareKmeans} presents a critical difference diagram \citep{CD} for the comparison based on Rand Index. %More detailed results are available in the supplementary materials. 
The values adjacent to each method represent the respective average rank (with smaller being better), and the methods connected by a thick bar do not significantly differ at the 95\% confidence level. Notably, there is no thick bar present in Figure \ref{fig:compareKmeans}, suggesting all methods have significant differences from each other.

\begin{table*}[]
\caption{Comparing RandomNet with the state-of-the-art methods on the UCR Archive benchmarks}
\label{tabe:compareSOA}
\scalebox{0.63}{\begin{tabular}{@{}llcccccccccc@{}}
\toprule
Datasets                     & Data Type  & \multicolumn{1}{c}{k-means} & \multicolumn{1}{c}{KSC} & \multicolumn{1}{c}{k-shape} & \multicolumn{1}{c}{SPF} & \multicolumn{1}{c}{SPIRAL} & \multicolumn{1}{c}{kDBA} & \multicolumn{1}{c}{IDEC} & \multicolumn{1}{c}{DTC} & \multicolumn{1}{c}{MiniR} & \multicolumn{1}{c}{Ours} \\ \midrule
ACSF1                        & Device     & 0.705                      & 0.736                   & 0.713                      & 0.862                   & 0.782                      & 0.726                    & 0.844                    & 0.437                   & \textbf{0.881}                 & 0.836                         \\
Adiac                        & Image      & 0.940                      & 0.947                   & 0.946                      & 0.961                   & 0.900                      & 0.955                    & 0.026                    & 0.786                   & \textbf{0.962}                 & 0.960                         \\
AllGestureWiimoteX           & Sensor     & 0.830                      & 0.099                   & \textbf{0.854}             & 0.836                   & 0.835                      & 0.819                    & 0.824                    & 0.817                   & 0.820                          & 0.834                         \\
AllGestureWiimoteY           & Sensor     & 0.825                      & 0.099                   & \textbf{0.867}             & 0.835                   & 0.842                      & 0.814                    & 0.818                    & 0.819                   & 0.829                          & 0.833                         \\
AllGestureWiimoteZ           & Sensor     & 0.830                      & 0.099                   & \textbf{0.843}             & 0.841                   & 0.763                      & 0.810                    & 0.819                    & 0.836                   & 0.811                          & 0.837                         \\
ArrowHead                    & Image      & 0.573                      & 0.630                   & 0.623                      & 0.644                   & 0.608                      & 0.531                    & 0.667                    & \textbf{0.694}          & 0.661                          & 0.667                         \\
Beef                         & Spectro    & 0.675                      & 0.714                   & 0.702                      & 0.722                   & 0.592                      & 0.677                    & 0.688                    & \textbf{0.748}          & 0.668                          & 0.733                         \\
BeetleFly                    & Image      & 0.508                      & 0.487                   & 0.591                      & 0.492                   & 0.519                      & 0.487                    & 0.488                    & 0.549                   & \textbf{0.771}                 & 0.621                         \\
BirdChicken                  & Image      & 0.533                      & 0.533                   & 0.569                      & \textbf{0.704}          & 0.499                      & 0.487                    & 0.550                    & 0.549                   & 0.544                          & 0.508                         \\
BME                          & Simulated  & 0.612                      & 0.565                   & 0.660                      & 0.637                   & 0.611                      & \textbf{0.775}           & 0.615                    & 0.705                   & 0.610                          & 0.621                         \\
Car                          & Sensor     & 0.675                      & 0.676                   & 0.631                      & \textbf{0.816}          & 0.642                      & 0.468                    & 0.674                    & 0.722                   & 0.657                          & 0.677                         \\
CBF                          & Simulated  & 0.704                      & 0.525                   & 0.877                      & \textbf{0.989}          & 0.717                      & 0.888                    & 0.671                    & 0.725                   & 0.750                          & 0.983                         \\
Chinatown                    & Traffic    & 0.527                      & 0.526                   & 0.526                      & 0.633                   & \textbf{0.787}             & 0.569                    & 0.513                    & 0.527                   & 0.582                          & 0.592                         \\
ChlorineConcentration        & Sensor     & 0.527                      & 0.529                   & 0.528                      & 0.535                   & 0.528                      & 0.532                    & 0.535                    & 0.506                   & 0.507                          & \textbf{0.536}                \\
CinCECGTorso                 & Sensor     & 0.677                      & 0.712                   & 0.626                      & 0.667                   & 0.752                      & 0.586                    & 0.669                    & \textbf{0.779}          & 0.670                          & 0.705                         \\
Coffee                       & Spectro    & 0.751                      & 0.805                   & 0.751                      & 0.834                   & 0.805                      & 0.777                    & 0.492                    & 0.492                   & 0.750                          & \textbf{1.000}                \\
Computers                    & Device     & 0.500                      & 0.499                   & 0.529                      & \textbf{0.550}          & 0.529                      & 0.523                    & 0.502                    & 0.511                   & 0.546                          & 0.532                         \\
CricketX                     & Motion     & 0.833                      & 0.368                   & \textbf{0.870}             & 0.867                   & 0.856                      & 0.828                    & 0.845                    & 0.846                   & 0.860                          & 0.868                         \\
CricketY                     & Motion     & 0.854                      & 0.514                   & \textbf{0.876}             & 0.871                   & 0.856                      & 0.805                    & 0.858                    & 0.848                   & 0.869                          & 0.868                         \\
CricketZ                     & Motion     & 0.857                      & 0.267                   & \textbf{0.873}             & 0.865                   & 0.858                      & 0.802                    & 0.853                    & 0.847                   & 0.860                          & 0.869                         \\
Crop                         & Image      & 0.934                      & 0.042                   & 0.930                      & 0.933                   & 0.916                      & 0.920                    & 0.042                    & 0.874                   & \textbf{0.941}                 & \textbf{0.941}                \\
DiatomSizeReduction          & Image      & 0.928                      & 0.928                   & \textbf{0.985}             & 0.911                   & 0.296                      & 0.951                    & 0.752                    & 0.920                   & 0.810                          & 0.945                         \\
DistalPhalanxOutlineAgeGroup & Image      & 0.610                      & 0.615                   & 0.725                      & 0.630                   & 0.741                      & \textbf{0.751}           & 0.465                    & 0.707                   & 0.620                          & 0.711                         \\
DistalPhalanxOutlineCorrect  & Image      & 0.499                      & 0.499                   & 0.499                      & 0.500                   & 0.501                      & 0.502                    & \textbf{0.526}           & 0.521                   & 0.501                          & 0.499                         \\
DistalPhalanxTW              & Image      & 0.873                      & 0.880                   & 0.770                      & 0.742                   & \textbf{0.901}             & 0.883                    & 0.297                    & 0.797                   & 0.777                          & 0.746                         \\
DodgerLoopDay                & Sensor     & 0.763                      & 0.767                   & 0.776                      & 0.799                   & 0.743                      & 0.766                    & 0.771                    & \textbf{0.806}          & 0.779                          & 0.794                         \\
DodgerLoopGame               & Sensor     & 0.503                      & \textbf{0.639}          & 0.585                      & 0.517                   & 0.529                      & 0.515                    & 0.556                    & 0.597                   & 0.502                          & 0.529                         \\
DodgerLoopWeekend            & Sensor     & \textbf{0.975}             & 0.963                   & 0.499                      & 0.707                   & 0.817                      & 0.950                    & 0.551                    & 0.622                   & 0.963                          & 0.715                         \\
Earthquakes                  & Sensor     & 0.501                      & 0.518                   & 0.502                      & 0.499                   & 0.500                      & 0.501                    & 0.499                    & 0.375                   & \textbf{0.861}                 & 0.500                         \\
ECG200                       & ECG        & 0.618                      & 0.613                   & 0.613                      & 0.604                   & 0.618                      & 0.556                    & \textbf{0.644}           & 0.517                   & 0.537                          & 0.604                         \\
ECG5000                      & ECG        & 0.749                      & 0.562                   & \textbf{0.797}             & 0.638                   & 0.787                      & 0.706                    & 0.468                    & 0.645                   & 0.625                          & 0.672                         \\
ECGFiveDays                  & ECG        & 0.500                      & 0.871                   & \textbf{0.879}             & 0.578                   & 0.530                      & 0.511                    & 0.506                    & 0.596                   & 0.761                          & 0.531                         \\
ElectricDevices              & Device     & 0.711                      & 0.598                   & 0.727                      & 0.773                   & 0.778                      & 0.781                    & 0.182                    & 0.790                   & \textbf{0.850}                 & 0.791                         \\
EOGHorizontalSignal          & EOG        & 0.857                      & 0.422                   & \textbf{0.877}             & 0.869                   & 0.866                      & 0.797                    & 0.082                    & 0.855                   & 0.569                          & 0.874                         \\
EOGVerticalSignal            & EOG        & 0.856                      & 0.603                   & \textbf{0.876}             & 0.870                   & 0.851                      & 0.800                    & 0.855                    & 0.838                   & 0.786                          & 0.861                         \\
EthanolLevel                 & Spectro    & 0.623                      & 0.621                   & 0.623                      & 0.626                   & 0.606                      & 0.553                    & 0.613                    & \textbf{0.689}          & 0.621                          & 0.627                         \\
FaceAll                      & Image      & 0.882                      & 0.899                   & 0.906                      & \textbf{0.922}          & \textbf{0.922}             & 0.858                    & 0.085                    & 0.847                   & 0.908                          & 0.903                         \\
FaceFour                     & Image      & 0.715                      & 0.393                   & 0.775                      & \textbf{1.000}          & 0.797                      & 0.305                    & 0.758                    & 0.786                   & 0.815                          & 0.844                         \\
FacesUCR                     & Image      & 0.879                      & 0.295                   & 0.913                      & \textbf{0.938}          & 0.914                      & 0.872                    & 0.886                    & 0.854                   & 0.909                          & 0.897                         \\
FiftyWords                   & Image      & 0.946                      & 0.757                   & 0.954                      & 0.955                   & 0.956                      & 0.812                    & 0.936                    & 0.880                   & \textbf{0.957}                 & 0.952                         \\
Fish                         & Image      & 0.786                      & 0.797                   & 0.777                      & \textbf{0.861}          & 0.711                      & 0.776                    & 0.764                    & 0.793                   & 0.823                          & 0.800                         \\
FordA                        & Sensor     & 0.500                      & 0.505                   & \textbf{0.578}             & 0.501                   & 0.500                      & 0.500                    & 0.500                    & 0.510                   & 0.500                          & 0.501                         \\
FordB                        & Sensor     & 0.500                      & 0.500                   & \textbf{0.527}             & 0.503                   & 0.500                      & 0.516                    & 0.501                    & 0.501                   & 0.526                          & 0.500                         \\
FreezerRegularTrain          & Sensor     & 0.644                      & 0.636                   & 0.639                      & \textbf{0.698}          & 0.639                      & 0.635                    & 0.500                    & 0.537                   & 0.641                          & 0.639                         \\
FreezerSmallTrain            & Sensor     & 0.645                      & 0.621                   & 0.639                      & \textbf{0.669}          & 0.639                      & 0.635                    & 0.500                    & 0.540                   & 0.642                          & 0.639                         \\
Fungi                        & HRM        & 0.938                      & 0.794                   & 0.798                      & 0.990                   & 0.993                      & 0.398                    & 0.959                    & 0.926                   & \textbf{0.999}                 & 0.976                         \\
GestureMidAirD1              & Trajectory & 0.933                      & 0.114                   & 0.939                      & 0.945                   & 0.945                      & 0.904                    & 0.934                    & 0.884                   & \textbf{0.953}                 & 0.939                         \\
GestureMidAirD2              & Trajectory & 0.907                      & 0.178                   & 0.912                      & 0.947                   & 0.932                      & 0.924                    & 0.925                    & 0.886                   & \textbf{0.948}                 & 0.945                         \\
GestureMidAirD3              & Trajectory & 0.918                      & 0.080                   & 0.927                      & 0.937                   & 0.927                      & 0.900                    & 0.932                    & 0.879                   & \textbf{0.940}                 & 0.920                         \\
GesturePebbleZ1              & Sensor     & 0.802                      & 0.213                   & 0.870                      & \textbf{0.904}          & 0.795                      & 0.750                    & 0.838                    & 0.841                   & 0.832                          & 0.798                         \\
GesturePebbleZ2              & Sensor     & 0.818                      & 0.165                   & 0.830                      & \textbf{0.900}          & 0.801                      & 0.722                    & 0.786                    & 0.876                   & 0.832                          & 0.795                         \\
GunPoint                     & Motion     & 0.497                      & 0.507                   & 0.497                      & 0.497                   & 0.498                      & 0.497                    & 0.498                    & \textbf{0.530}          & 0.497                          & 0.497                         \\
GunPointAgeSpan              & Motion     & \textbf{0.628}             & 0.518                   & 0.530                      & 0.514                   & 0.546                      & 0.518                    & 0.499                    & 0.559                   & 0.499                          & 0.499                         \\
GunPointMaleVersusFemale     & Motion     & 0.500                      & 0.737                   & \textbf{0.796}             & 0.649                   & 0.500                      & 0.750                    & 0.513                    & 0.580                   & 0.499                          & 0.617                         \\
GunPointOldVersusYoung       & Motion     & 0.589                      & 0.506                   & 0.518                      & 0.500                   & 0.509                      & 0.519                    & 0.507                    & 0.534                   & \textbf{1.000}                 & 0.619                         \\
Ham                          & Spectro    & \textbf{0.532}             & \textbf{0.532}          & \textbf{0.532}             & 0.510                   & 0.521                      & 0.506                    & 0.498                    & 0.514                   & 0.520                          & 0.512                         \\
HandOutlines                 & Image      & 0.673                      & \textbf{0.684}          & 0.679                      & 0.558                   & 0.575                      & 0.581                    & 0.665                    & 0.518                   & 0.539                          & 0.581                         \\
Haptics                      & Motion     & 0.693                      & 0.700                   & 0.700                      & 0.718                   & 0.612                      & 0.618                    & 0.700                    & \textbf{0.735}          & 0.683                          & 0.703                         \\
Herring                      & Image      & 0.500                      & 0.499                   & 0.504                      & 0.504                   & 0.506                      & 0.499                    & 0.504                    & 0.489                   & 0.502                          & \textbf{0.508}                \\
HouseTwenty                  & Device     & 0.592                      & 0.498                   & 0.521                      & 0.527                   & 0.576                      & \textbf{0.742}           & 0.521                    & 0.515                   & 0.675                          & 0.534                         \\
InlineSkate                  & Motion     & 0.736                      & 0.760                   & 0.749                      & 0.759                   & 0.693                      & 0.669                    & 0.749                    & \textbf{0.763}          & 0.738                          & 0.759                         \\
InsectEPGRegularTrain        & EPG        & 0.639                      & 0.523                   & 0.709                      & 0.718                   & 0.767                      & 0.648                    & 0.596                    & 0.613                   & 0.753                          & \textbf{1.000}                \\
InsectEPGSmallTrain          & EPG        & 0.564                      & 0.574                   & 0.707                      & 0.732                   & 0.772                      & 0.629                    & 0.628                    & 0.635                   & 0.775                          & \textbf{1.000}                \\
InsectWingbeatSound          & Sensor     & \textbf{0.887}             & 0.861                   & 0.835                      & 0.879                   & 0.886                      & 0.753                    & 0.873                    & 0.883                   & 0.885                          & 0.873                         \\
ItalyPowerDemand             & Sensor     & 0.500                      & 0.505                   & \textbf{0.781}             & 0.635                   & 0.502                      & 0.501                    & 0.777                    & 0.509                   & 0.500                          & 0.552                         \\
LargeKitchenAppliances       & Device     & 0.564                      & 0.386                   & 0.634                      & 0.568                   & 0.575                      & 0.627                    & 0.558                    & \textbf{0.639}          & 0.557                          & 0.604                         \\ \bottomrule
\end{tabular}}
\end{table*}

\begin{table*}[]
\caption{(Continue) Comparing RandomNet with the state-of-the-art methods on the UCR Archive benchmarks\\ \dag \ indicates that this dataset is used for the hyperparameter selection}
\label{tabe:compareSOA2}
\scalebox{0.63}{\begin{tabular}{@{}llcccccccccc@{}}
\toprule
Datasets                                   & Data Type    & k-means & KSC            & k-shape         & SPF            & SPIRAL         & kDBA           & IDEC           & DTC            & MiniR          & Ours           \\ \midrule
Lightning2                                 & Sensor       & 0.499  & 0.497          & 0.517          & 0.538          & \textbf{0.560} & 0.508          & 0.517          & 0.536          & 0.523          & 0.533          \\
Lightning7                                 & Sensor       & 0.790  & 0.555          & 0.794          & 0.809          & 0.796          & 0.796          & 0.789          & 0.805          & 0.812          & \textbf{0.814} \\
Mallat                                     & Simulated    & 0.968  & 0.960          & 0.920          & 0.949          & 0.985          & 0.930          & 0.921          & 0.923          & 0.962          & \textbf{0.989} \\
Meat                                       & Spectro      & 0.785  & 0.785          & 0.729          & 0.830          & 0.852          & 0.800          & 0.328          & 0.578          & 0.768          & \textbf{0.860} \\
MedicalImages                              & Image        & 0.665  & 0.556          & 0.668          & 0.677          & 0.668          & \textbf{0.684} & 0.682          & 0.651          & 0.678          & 0.677          \\
MelbournePedestrian                        & Traffic      & 0.881  & 0.000          & 0.861          & 0.876          & 0.864          & 0.848          & 0.100          & 0.436          & \textbf{0.909} & 0.876          \\
MiddlePhalanxOutlineAgeGroup               & Image        & 0.733  & 0.733          & 0.733          & 0.701          & 0.733          & 0.630          & 0.382          & 0.719          & 0.720          & \textbf{0.735} \\
MiddlePhalanxOutlineCorrect                & Image        & 0.500  & 0.500          & 0.500          & 0.500          & 0.499          & 0.499          & \textbf{0.529} & 0.503          & 0.499          & 0.506          \\
MiddlePhalanxTW                            & Image        & 0.738  & 0.774          & 0.785          & 0.770          & 0.772          & 0.771          & 0.765          & 0.393          & 0.764          & \textbf{0.786} \\
MixedShapesRegularTrain                    & Image        & 0.819  & 0.720          & 0.789          & 0.802          & 0.817          & 0.696          & 0.208          & 0.825          & \textbf{0.848} & 0.824          \\
MixedShapesSmallTrain                      & Image        & 0.820  & 0.730          & 0.790          & 0.774          & 0.816          & 0.690          & 0.211          & 0.826          & \textbf{0.847} & 0.825          \\
MoteStrain                                 & Sensor       & 0.718  & 0.563          & 0.803          & 0.698          & 0.818          & 0.518          & 0.630          & 0.661          & \textbf{0.852} & 0.759          \\
NonInvasiveFetalECGThorax1                 & ECG          & 0.952  & 0.953          & 0.958          & 0.968          & 0.952          & 0.958          & 0.933          & 0.824          & \textbf{0.976} & 0.975          \\
NonInvasiveFetalECGThorax2                 & ECG          & 0.961  & 0.964          & 0.962          & 0.971          & 0.957          & 0.953          & 0.945          & 0.865          & 0.978          & \textbf{0.981} \\
OliveOil                                   & Spectro      & 0.739  & 0.845          & 0.745          & \textbf{0.875} & 0.872          & 0.828          & 0.288          & 0.288          & 0.775          & 0.811          \\
OSULeaf                                    & Image        & 0.752  & 0.455          & 0.791          & \textbf{0.793} & 0.754          & 0.716          & 0.731          & 0.789          & 0.691          & 0.763          \\
PhalangesOutlinesCorrect                   & Image        & 0.505  & 0.506          & 0.505          & 0.504          & 0.504          & 0.502          & \textbf{0.539} & 0.511          & 0.501          & 0.506          \\
Phoneme                                    & Sensor       & 0.911  & 0.491          & 0.929          & 0.930          & 0.928          & 0.789          & 0.922          & 0.082          & 0.928          & \textbf{0.932} \\
PickupGestureWiimoteZ                      & Sensor       & 0.846  & 0.233          & 0.887          & 0.888          & 0.836          & 0.889          & 0.860          & 0.866          & 0.886          & \textbf{0.895} \\
PigAirwayPressure                          & Hemodynamics & 0.914  & 0.016          & 0.903          & 0.936          & 0.960          & 0.840          & 0.938          & 0.883          & 0.967          & \textbf{0.969} \\
PigArtPressure                             & Hemodynamics & 0.965  & 0.889          & 0.958          & 0.969          & 0.954          & 0.722          & 0.964          & 0.892          & \textbf{0.986} & 0.973          \\
PigCVP                                     & Hemodynamics & 0.944  & 0.603          & 0.962          & \textbf{0.967} & 0.938          & 0.857          & 0.946          & 0.882          & \textbf{0.967} & 0.964          \\
PLAID                                      & Device       & 0.816  & 0.122          & 0.772          & \textbf{0.841} & 0.819          & 0.809          & 0.823          & 0.799          & 0.779          & 0.775          \\
Plane                                      & Sensor       & 0.872  & 0.945          & 0.920          & 0.992          & 0.952          & 0.918          & 0.892          & 0.946          & \textbf{0.995} & \textbf{0.995} \\
PowerCons                                  & Power        & 0.547  & 0.512          & 0.590          & 0.572          & 0.501          & 0.502          & 0.707          & 0.567          & 0.760          & \textbf{0.768} \\
ProximalPhalanxOutlineAgeGroup             & Image        & 0.690  & 0.780          & 0.780          & 0.696          & \textbf{0.806} & 0.688          & 0.390          & 0.770          & 0.700          & 0.792          \\
ProximalPhalanxOutlineCorrect              & Image        & 0.534  & 0.533          & 0.535          & 0.523          & 0.529          & 0.525          & \textbf{0.564} & 0.519          & 0.524          & 0.527          \\
ProximalPhalanxTW                          & Image        & 0.758  & 0.790          & \textbf{0.862} & 0.743          & 0.799          & 0.787          & 0.786          & 0.289          & 0.766          & 0.763          \\
RefrigerationDevices                       & Device       & 0.555  & 0.332          & 0.556          & 0.558          & 0.587          & \textbf{0.588} & 0.554          & 0.518          & 0.538          & 0.580          \\
Rock                                       & Spectrum     & 0.664  & 0.380          & 0.689          & 0.719          & 0.657          & 0.675          & 0.660          & 0.722          & \textbf{0.734} & 0.695          \\
ScreenType                                 & Device       & 0.562  & 0.332          & 0.557          & 0.568          & 0.566          & 0.525          & 0.562          & \textbf{0.635} & 0.559          & 0.569          \\
SemgHandGenderCh2                          & Spectrum     & 0.552  & 0.549          & 0.546          & 0.513          & 0.502          & \textbf{0.561} & 0.526          & 0.508          & 0.501          & 0.514          \\
SemgHandMovementCh2                        & Spectrum     & 0.735  & 0.638          & 0.739          & 0.756          & 0.604          & 0.732          & 0.743          & \textbf{0.783} & 0.762          & 0.743          \\
SemgHandSubjectCh2                         & Spectrum     & 0.734  & 0.645          & 0.721          & 0.734          & 0.568          & 0.661          & 0.730          & \textbf{0.797} & 0.698          & 0.700          \\
ShakeGestureWiimoteZ                       & Sensor       & 0.863  & 0.249          & 0.909          & \textbf{0.932} & 0.870          & 0.883          & 0.856          & 0.875          & 0.905          & 0.908          \\
ShapeletSim                                & Simulated    & 0.502  & 0.498          & 0.691          & 0.518          & 0.498          & 0.499          & 0.498          & 0.506          & 0.500          & \textbf{0.703} \\
ShapesAll                                  & Image        & 0.958  & 0.631          & 0.976          & \textbf{0.982} & 0.972          & 0.881          & 0.953          & 0.870          & \textbf{0.982} & 0.981          \\
SmallKitchenAppliances                     & Device       & 0.558  & 0.523          & 0.506          & 0.636          & 0.645          & 0.582          & 0.556          & 0.641          & 0.602          & \textbf{0.665} \\
SmoothSubspace                             & Simulated    & 0.709  & 0.333          & 0.682          & 0.645          & \textbf{0.896} & 0.629          & 0.585          & 0.638          & 0.631          & 0.816          \\
SonyAIBORobotSurface1                      & Sensor       & 0.800  & 0.825          & 0.501          & 0.732          & 0.756          & 0.860          & 0.507          & 0.623          & \textbf{0.907} & 0.830          \\
SonyAIBORobotSurface2                      & Sensor       & 0.662  & \textbf{0.705} & 0.526          & 0.699          & 0.647          & 0.550          & 0.555          & 0.574          & 0.690          & 0.681          \\
StarLightCurves                            & Sensor       & 0.770  & 0.769          & 0.767          & 0.717          & 0.774          & \textbf{0.782} & 0.764          & 0.643          & 0.751          & 0.720          \\
Strawberry \dag             & Spectro      & 0.504  & 0.504          & 0.504          & 0.529          & 0.506          & 0.503          & 0.515          & 0.515          & 0.505          & \textbf{0.533} \\
SwedishLeaf \dag            & Image        & 0.879  & 0.911          & 0.910          & 0.919          & 0.831          & 0.879          & 0.888          & 0.801          & \textbf{0.940} & 0.930          \\
Symbols \dag                & Image        & 0.900  & 0.610          & 0.928          & 0.970          & 0.898          & 0.885          & 0.899          & 0.898          & \textbf{0.987} & 0.919          \\
SyntheticControl \dag       & Simulated    & 0.865  & 0.507          & 0.901          & \textbf{0.986} & 0.886          & 0.941          & 0.821          & 0.885          & 0.962          & 0.981          \\
ToeSegmentation1 \dag       & Motion       & 0.499  & 0.534          & 0.505          & \textbf{0.602} & 0.499          & 0.545          & 0.498          & 0.502          & 0.500          & 0.592          \\
ToeSegmentation2 \dag       & Motion       & 0.498  & 0.558          & \textbf{0.665} & 0.586          & 0.501          & 0.563          & 0.507          & 0.421          & 0.616          & 0.563          \\
Trace \dag                  & Sensor       & 0.750  & 0.744          & 0.246          & \textbf{0.981} & 0.750          & 0.860          & 0.750          & 0.751          & 0.751          & 0.749          \\
TwoLeadECG \dag             & ECG          & 0.502  & 0.543          & 0.540          & \textbf{0.650} & 0.502          & 0.558          & 0.500          & 0.506          & 0.503          & 0.502          \\
TwoPatterns \dag            & Simulated    & 0.628  & 0.537          & 0.675          & 0.693          & 0.656          & \textbf{0.945} & 0.630          & 0.705          & 0.638          & 0.625          \\
UMD \dag                    & Simulated    & 0.557  & 0.559          & 0.612          & 0.622          & 0.626          & 0.557          & 0.614          & \textbf{0.682} & 0.616          & 0.621          \\
UWaveGestureLibraryAll \dag & Motion       & 0.897  & 0.705          & 0.910          & 0.852          & 0.897          & 0.800          & 0.914          & \textbf{0.927} & 0.905          & 0.864          \\
UWaveGestureLibraryX \dag   & Motion       & 0.856  & 0.413          & 0.857          & 0.858          & 0.863          & 0.849          & 0.856          & \textbf{0.872} & 0.862          & 0.857          \\
UWaveGestureLibraryY \dag   & Motion       & 0.849  & 0.579          & 0.830          & 0.831          & 0.848          & 0.823          & 0.848          & \textbf{0.857} & 0.853          & 0.843          \\
UWaveGestureLibraryZ \dag   & Motion       & 0.845  & 0.517          & 0.852          & 0.844          & 0.845          & 0.748          & 0.849          & \textbf{0.857} & 0.845          & 0.849          \\
Wafer \dag                  & Sensor       & 0.535  & 0.509          & 0.537          & 0.505          & 0.534          & 0.534          & \textbf{0.538} & 0.298          & 0.534          & 0.508          \\
Wine \dag                   & Spectro      & 0.496  & 0.496          & 0.496          & 0.500          & 0.495          & 0.496          & 0.496          & 0.496          & \textbf{0.503} & 0.502          \\
WordSynonyms \dag           & Image        & 0.897  & 0.765          & 0.894          & 0.901          & 0.903          & 0.783          & 0.885          & 0.848          & \textbf{0.904} & 0.901          \\
Worms \dag                  & Motion       & 0.646  & 0.520          & 0.656          & 0.683          & 0.666          & 0.644          & 0.263          & \textbf{0.701} & 0.648          & 0.687          \\
WormsTwoClass \dag          & Motion       & 0.499  & 0.509          & 0.506          & 0.523          & 0.499          & 0.498          & 0.510          & 0.492          & 0.521          & \textbf{0.529} \\
Yoga \dag                   & Image        & 0.500  & 0.500          & 0.500          & 0.500          & 0.500          & 0.502          & 0.503          & \textbf{0.504} & 0.500          & 0.501          
\\ \midrule
Average Rand Index                         &              & 0.718  & 0.561          & 0.732          & 0.742          & 0.723          & 0.694          & 0.630          & 0.680          & 0.741          & \textbf{0.750}          \\
Average rank                               &              & 6.000  & 7.074          & 4.741          & 4.023          & 5.287          & 6.690          & 6.810          & 5.898          & 4.745          & \textbf{3.732}         \\\bottomrule
\end{tabular}}
\end{table*}

\begin{table*}[]
\caption{Comparison of RandomNet with state-of-the-art methods across seven different time series data types. The values represent the average rank of each method for the respective data type. The summary provides the average rank for rankings and the counts for the top-1 and top-3.}\label{tab:data_type}
\scalebox{0.78}{
\begin{tabular}{@{}lcccccccccc@{}}
\toprule
Data Type   & \multicolumn{1}{l}{k-means} & \multicolumn{1}{c}{KSC} & \multicolumn{1}{c}{k-shape} & \multicolumn{1}{c}{SPF} & \multicolumn{1}{c}{SPIRAL} & \multicolumn{1}{c}{kDBA} & \multicolumn{1}{c}{IDEC} & \multicolumn{1}{c}{DTC} & \multicolumn{1}{c}{MiniR} & \multicolumn{1}{c}{Ours} \\ \midrule
Device      & 6.500                      & 9.333                   & 6.778                      & 3.667                   & 3.833                      & 4.667                    & 6.556                    & 5.444                   & 4.667                          & \textbf{3.556}                \\
ECG/EOG/EPG & 6.111                      & 7.056                   & 3.111                      & 4.167                   & 4.556                      & 6.722                    & 7.778                    & 7.444                   & 5.000                          & \textbf{3.056}                \\
Image       & 6.143                      & 6.464                   & 4.321                      & 4.946                   & 5.054                      & 6.839                    & 6.464                    & 6.143                   & 4.804                          & \textbf{3.821}                \\
Motion      & 6.222                      & 6.278                   & \textbf{3.556}             & 4.667                   & 6.333                      & 7.333                    & 6.333                    & 4.111                   & 5.944                          & 4.222                         \\
Sensor      & 5.982                      & 6.714                   & 4.786                      & \textbf{3.625}          & 5.500                      & 6.929                    & 6.946                    & 5.696                   & 4.607                          & 4.214                         \\
Simulated   & 5.200                      & 8.800                   & 4.600                      & 3.800                   & 5.400                      & 5.200                    & 8.400                    & 5.200                   & 6.200                          & \textbf{2.200}                \\
Spectro     & 5.750                      & 4.083                   & 5.500                      & 3.333                   & 5.083                      & 6.500                    & 8.833                    & 6.000                   & 6.917                          & \textbf{3.000}                \\ \midrule
Average rank    & 6.214                      & 7.786                   & 3.714                      & 2.571                   & 5.214                      & 7.643                    & 9.143                    & 5.929                   & 5.357                          & \textbf{1.429}                \\ 
Num.Top-1   & 0                          & 0                       & 1                          & 1                       & 0                          & 0                        & 0                        & 0                       & 0                              & \textbf{5}                             \\
Num.Top-3   & 0                          & 1                       & 4                          & 5                       & 1                          & 0                        & 0                        & 1                       & 2                              & \textbf{7}                             \\ \bottomrule
\end{tabular}}
\end{table*}

As seen in the figure, RandomNet significantly outperforms both k-means and kmeansE. It is noteworthy that kmeansE is significantly better than the standard k-means, indicating that employing ensemble methods can substantially improve the performance of time series clustering, even for the naive method that uses the original representation of time series. Comparing RandomNet with kmeansE further demonstrates that using the proposed deep neural network with random parameters for generating representations can indeed enhance the accuracy of k-means clustering and ensembles.

\begin{figure}[]
\centering
\includegraphics[width=0.93\textwidth, height=0.27\textwidth]{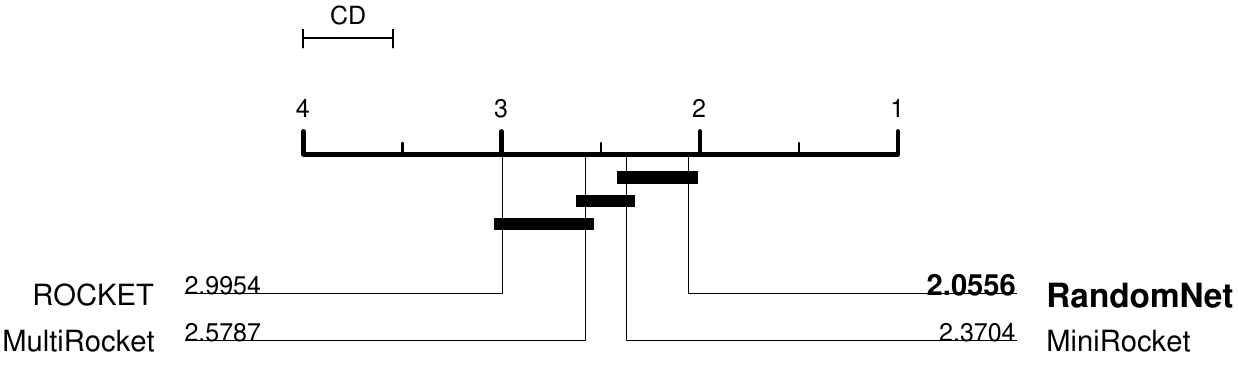}
\caption{Critical difference diagram of the comparison with ROCKET, MiniRocket and MultiRocket.}
\label{fig:rockets}
\end{figure}

\textbf{Comparison with ROCKET and its variants.}

We compare our method with ROCKET and its variants, MiniRocket and MultiRocket. Note that we remove the classifier components in these ROCKET variants and replace them with k-means to adapt them to our use case. As shown in Figure \ref{fig:rockets}, RandomNet outperforms ROCKET variants in terms of average rank, and is especially significantly better than ROCKET and MultiRocket. This reflects the superiority of RandomNet, which is specially designed for time series clustering, in models based on random parameters. It is worth noting that MiniRocket is the best model among ROCKET variants. Therefore, we keep only MiniRocket in subsequent experiments.

\textbf{Comparison with the state-of-the-arts.}
Table \ref{tabe:compareSOA} and \ref{tabe:compareSOA2} present the experimental results of RandomNet compared to state-of-the-art methods. The best results for each dataset are highlighted in bold. We provide the average Rand Index and average rank for each method. The $\dag$ symbol indicates that the dataset is used for hyperparameter selection. Consequently, the results from these datasets have not been included in the computation of the average Rand Index and average rank. As the results illustrate, RandomNet achieves the highest average Rand Index and average rank amongst all the baseline methods.

Figure \ref{fig:compareSOA2} depicts the critical difference diagram of the comparison between RandomNet and the state-of-the-art methods. The figure demonstrates that RandomNet significantly outperforms k-means, KSC, kDBA, SPIRAL, and two deep learning-based methods, IDEC and DTC. It also shows our proposed method is slightly better than MiniRocket, k-shape and SPF. These results solidify that RandomNet is a state-of-the-art time series clustering method.

In order to gain insights on the strengths and weaknesses of all the 10 methods compared and see how each method performs on different types of data, we divide the time series datasets into different categories (e.g. sensor, device, motion, spectro, etc). For each dataset, we rank the results of the 10 methods as we did in previous comparisons (1 is the best and 10 is the worst), and compute the average ranking of each method for each category. The ranking results are shown in Table \ref{tab:data_type}. Note we only include categories with at least five datasets. The best-performing method for each category is in bold. Next, we rank these average rankings for each category (e.g. for Device, our method has the ranking of 1 since it has the lowest average rank, whereas KSC has the ranking of 10 since it has the highest average rank). We then average these category-wise rankings and report them in the line \textit{Average rank}). For example, the rankings of our method for the 7 categories are: 1, 1, 1, 3, 2, 1, 1, respectively, with an average rank of 1.429. The last two lines show the numbers of categories in which the Rand index of a method is among top-1 and top-3, respectively.

Our model achieves the best average rank and is the best in five data types. It is also among the top three in all data types. This demonstrates the superiority of our model compared to other models across diverse data types. This can be attributed to its ability to generate diverse representations and its ensemble mechanism, which effectively cancels out irrelevant representations.

In contrast, other methods exhibit varying performance due to their specific focuses such as local shape or point-to-point distance computation, which may limit their effectiveness to only work on certain data types. For example, k-shape ranks ninth on the device data (where RandomNet ranks first), and SPF achieves an intermediate rank (fourth) on both image and motion data (where RandomNet ranks first and third, respectively). These results indicate that while specific models may perform well in certain data types, their performance can be suboptimal in others due to focus limitations.

\begin{figure}[]
\centering
\includegraphics[width=0.9\textwidth, height=0.4\textwidth]{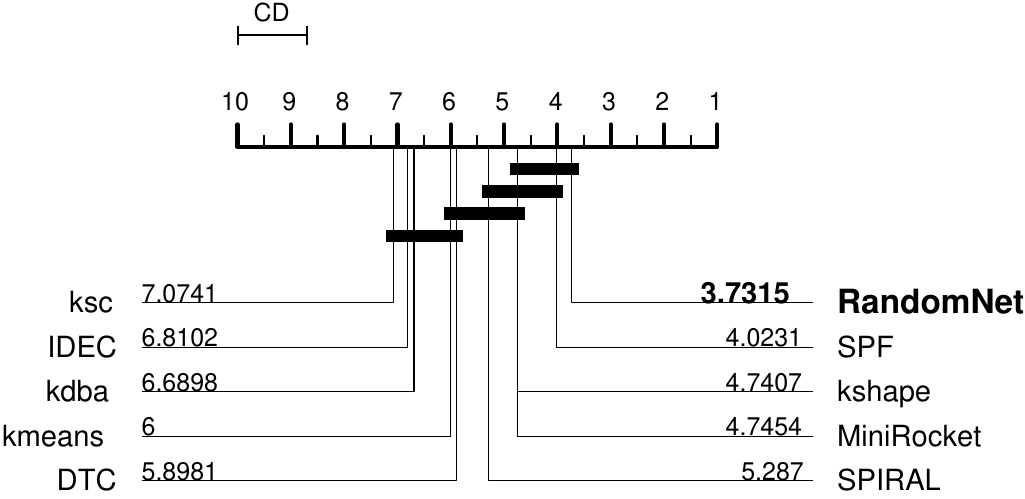}
\caption{Critical difference diagram of the comparison with state-of-the-art methods.}
\label{fig:compareSOA2}
\end{figure}

\begin{table}[]
\centering
\caption{Ablation results of RandomNet on 108 UCR datasets.}
\label{tab:ablation_study}
\begin{tabular}{lcc}
\hline
\textbf{Model Configuration} & \textbf{Average Rand Index} & \textbf{Average Rank} \\
\hline
RandomNet & \textbf{0.750} & \textbf{2.602} \\
RandomNet w/ GRU & 0.746 & 2.894 \\
RandomNet w/o LSTM & 0.747 & 2.880 \\
RandomNet w/o LSTM \& ReLU & 0.746 & 2.981 \\
RandomNet w/o LSTM \& ReLU \& Pooling & 0.738 & 3.644 \\
\hline
\end{tabular}
\label{tab:ablation}
\end{table}

\subsection{Ablation study}
To verify the effectiveness of each component in RandomNet, we compare the performance of full RandomNet and its four variants on 108 UCR datasets, which are shown in Table~\ref{tab:ablation}. The four variants are, 1) RandomNet w/ GRU (replaces LSTM with GRU), 2) RandomNet w/o LSTM (removes LSTM), 3) RandomNet w/o LSTM \& ReLU (removes LSTM and ReLU), and 4) RandomNet w/o LSTM \& ReLU \& pooling (removes LSTM, ReLU and pooling).

The results show that full RandomNet is better than the four variants in average rand index and average rank, reflecting the effectiveness of each part of RandomNet. It is worth noting that pooling is important in the model. Removing pooling will significantly increase the running time and decrease the performance.

\subsection{Visualizing clusters for different methods}
Figure \ref{fig:tsne} shows the 2D embeddings of the Cylinder-Bell-Funnel (CBF) \citep{CBF} dataset using t-distributed Stochastic Neighbor Embedding (t-SNE) algorithm \citep{tsne}, as well as cluster assignments by k-means, MiniRocket, and RandomNet compared with the true labels. We can see clearly that k-means and MiniRocket both have difficulty distinguishing the blue and green classes, which correspond to the Bell and the Cylinder classes, respectively.

Upon closer examination, we can see why. Figure \ref{fig:dend} shows five instances of the CBF time series and their cluster assignments from k-means, MiniRocket, and RandomNet, respectively. All methods successfully group the red time series (Funnel) into one cluster. However, k-means and MiniRocket inaccurately cluster the blue (Bell) and green (Cylinder) time series, whereas RandomNet is able to identify the correct clusters. This is due to k-means' sensitivity to misalignment in the time series data (e.g. the blue time series), high dimensionality, and noise as it clusters based on Euclidean distances. For MiniRocket, the use of a network with random weights results in many class-independent values in its final representation, which is equivalent to adding noise from its last layer to k-means. In contrast, RandomNet uses the selection mechanism and ensemble, which weakens the influence of irrelevant representation and strengthens relevant representation, making the model more robust.

\begin{figure}[]
\centering
\includegraphics[width=0.7\textwidth]{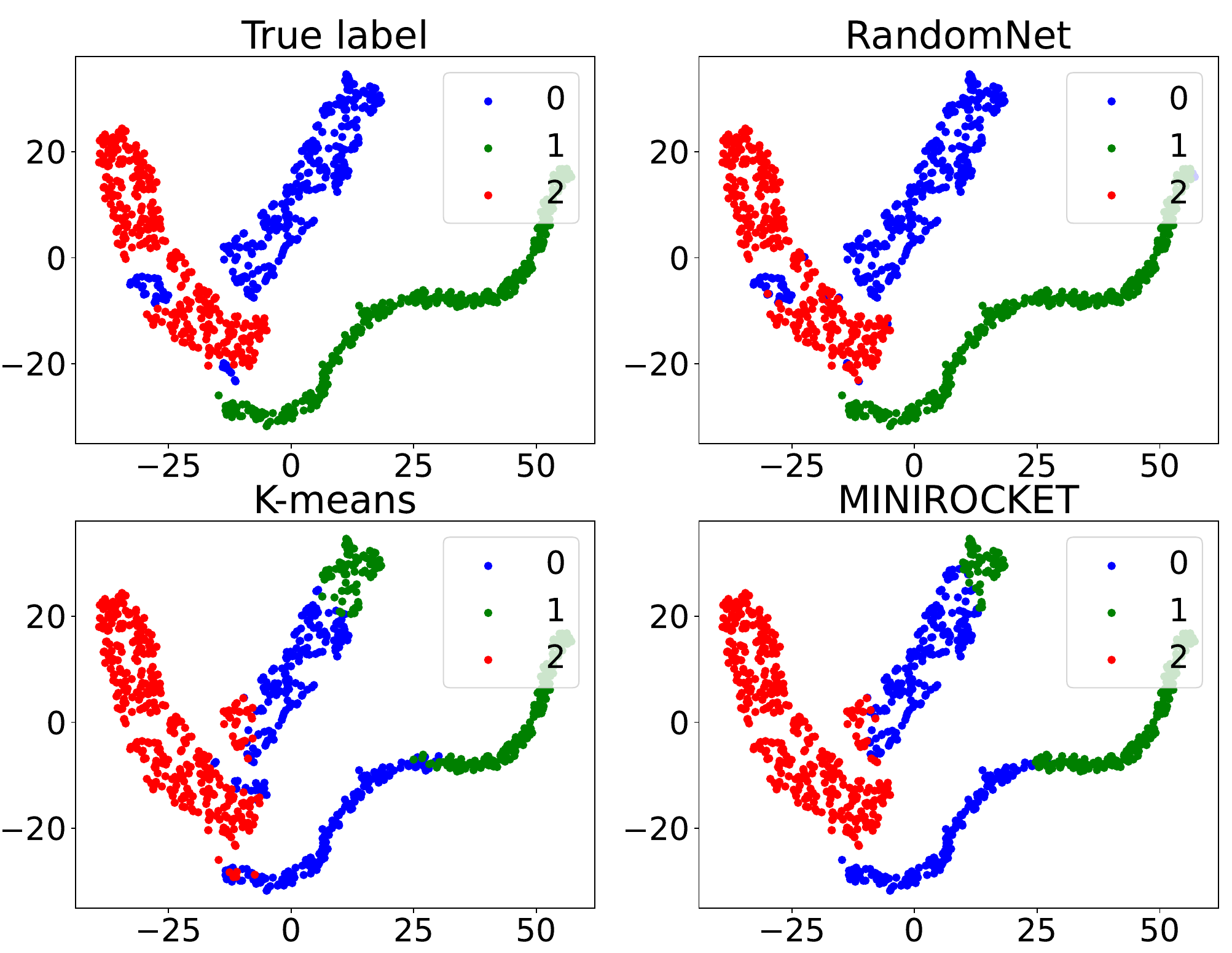}
\caption{Clusterings of CBF dataset visualized using t-SNE for RandomNet (upper right), k-means (lower left) and MiniRocket (lower right), compared with True Label (upper left).}
\label{fig:tsne}
\end{figure}

\begin{figure}[]
\centering
\includegraphics[width=0.85\textwidth]{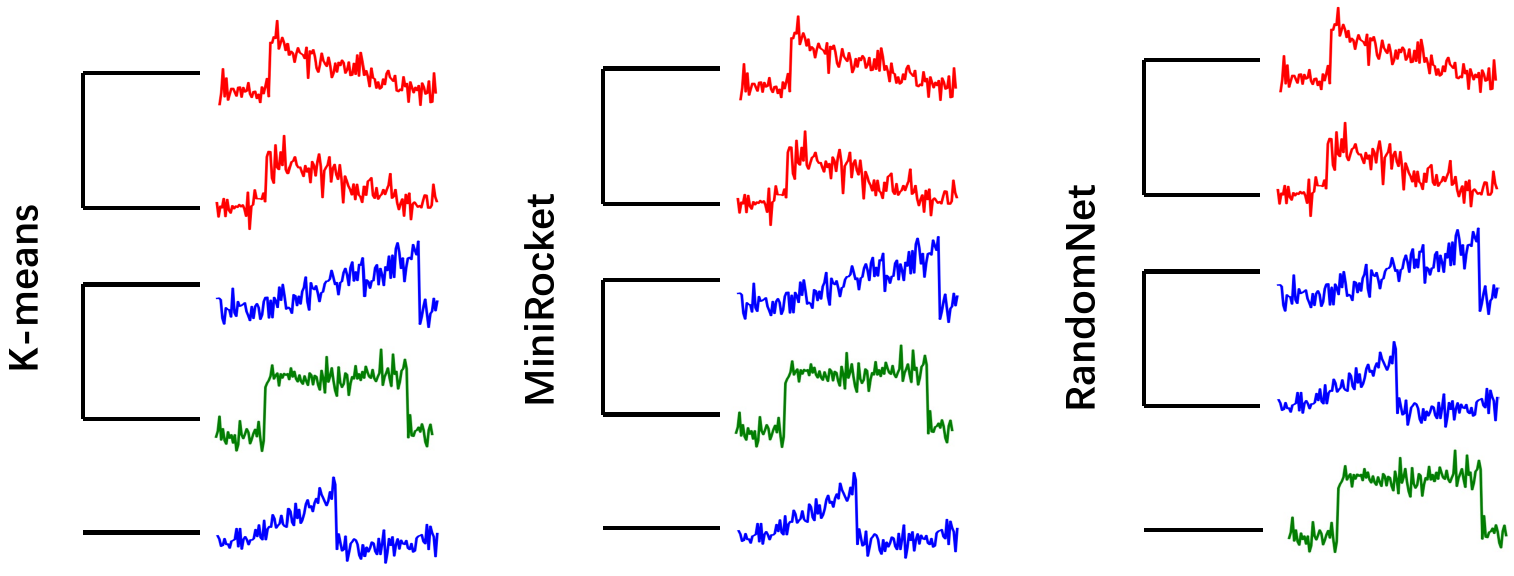}
\caption{Clustering results on five samples from the CBF dataset using k-means (left), MiniRocket (middle) and RandomNet (right).}
\label{fig:dend}
\end{figure}

\subsection{Testing the time complexity}\label{sec:time}

In real-world applications, the size of datasets and the length of time series can be huge, making linear time complexity with respect to the number of instances and length of time series an essential characteristic of any practical model. To test the scalability and effectiveness of our proposed method, we use the same mechanism to generate datasets of varying sizes. For different time series lengths, we supplement the original time series (length of 128) with random noise to reach the required length. In this experiment, we use the CBF dataset \citep{CBF}. For testing linear complexity w.r.t the number of instances, the number of instances is set from 200 to 10,000 with a fixed time series length of 100. For testing linear complexity w.r.t the length of time series, the length is set from 1,000 to 10,000 with a fixed dataset size of 120. We run RandomNet 10 times and record the average running time and Rand Index. The outcomes are presented in Figure \ref{fig:runingTime}.

\begin{figure}[]
    \centering
    \begin{subfigure}[b]{0.49\textwidth}
        \centering
        \includegraphics[width=\textwidth]{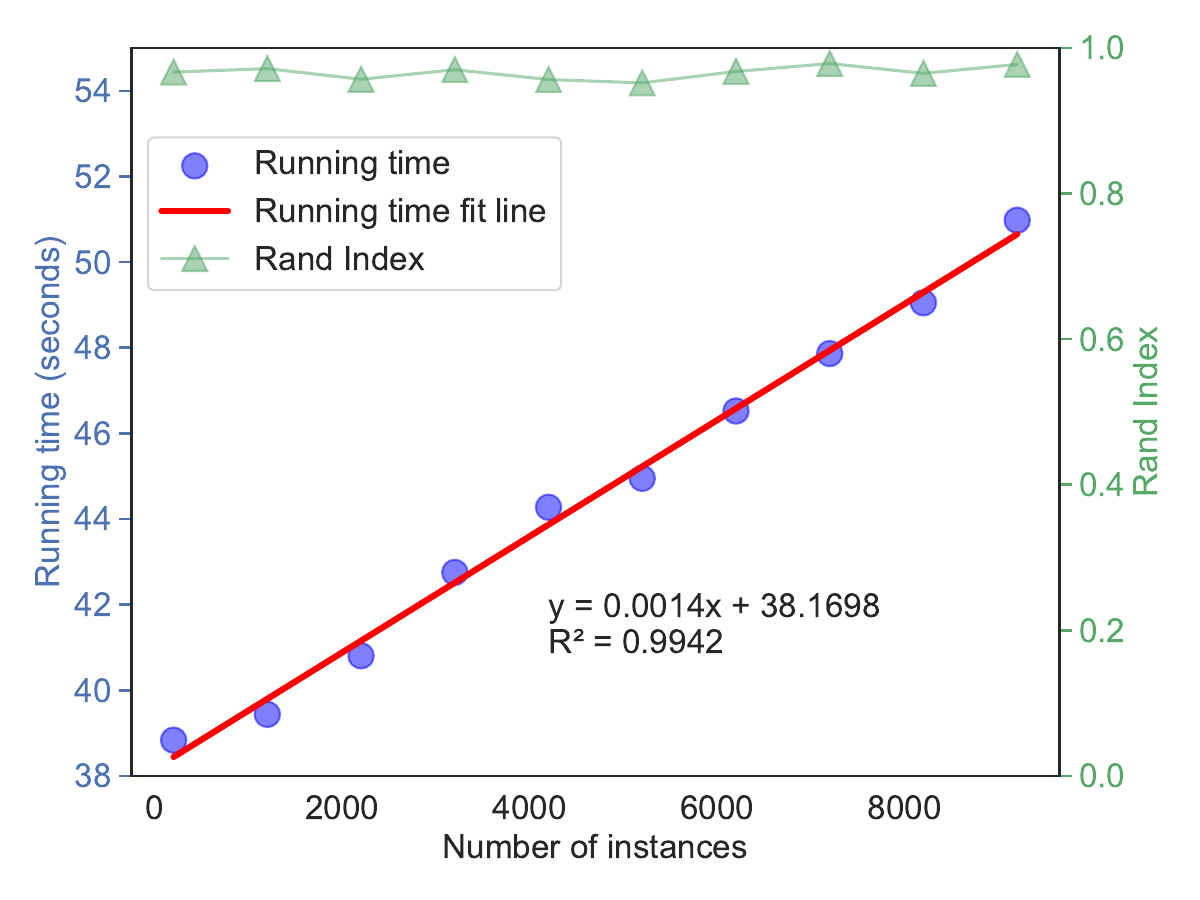}
    \end{subfigure}
    \hfill
    \begin{subfigure}[b]{0.49\textwidth}
        \centering
        \includegraphics[width=\textwidth]{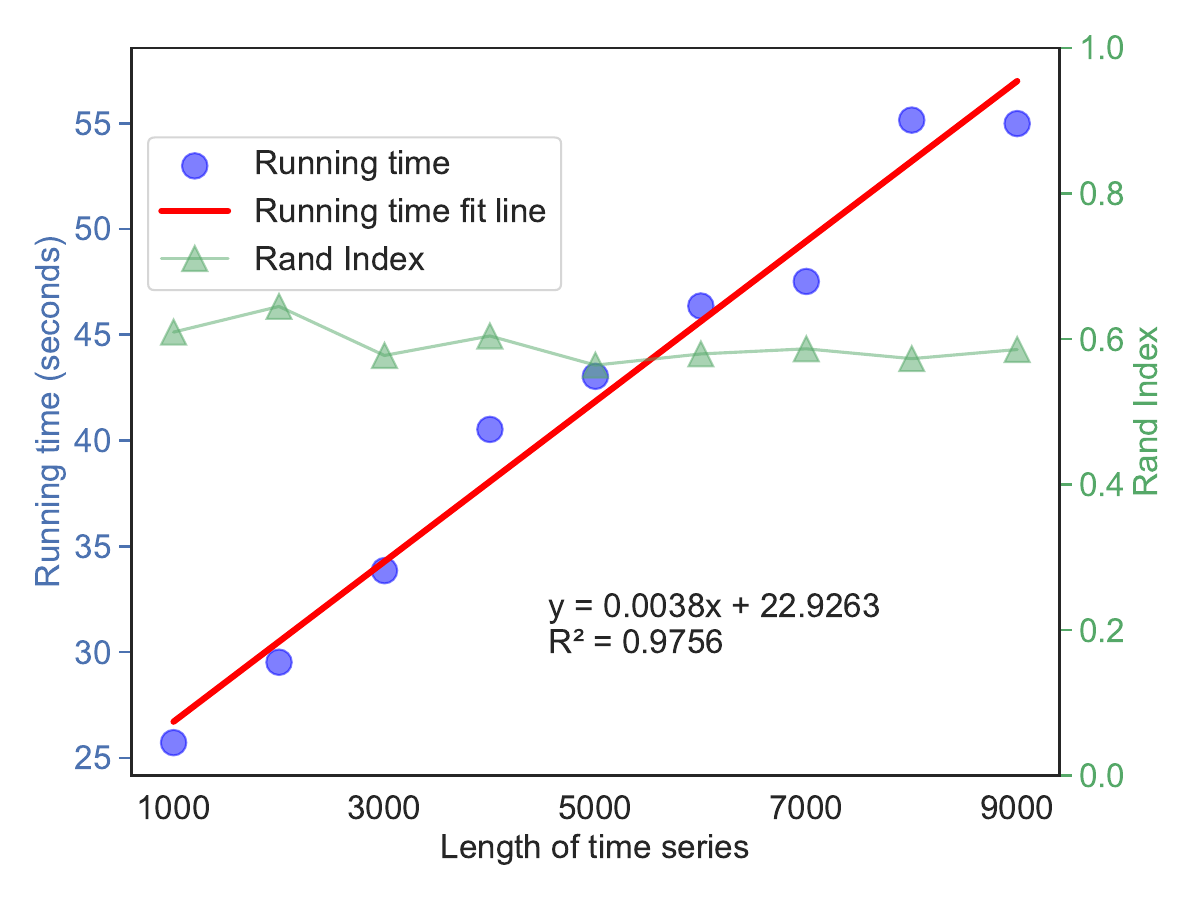}
    \end{subfigure}
    \caption{Running time of RandomNet on the different number of instances (left) and lengths of time series (right).}
    \label{fig:runingTime}
\end{figure}

In the figure, each blue dot represents the average running time corresponding to the respective number of instances or length of time series. We perform linear curve-fitting on the results, depicted by the red line. One can see from the figure that the $R^2$ value, which is the coefficient of determination of the fitting, is 0.9942 and 0.9756, respectively. The value is close to 1, indicating that the average running time of RandomNet has a strong linear relationship with the number of instances and length of time series. Moreover, we also observe stable Rand Index results across varying input sizes, indicating that our model is not sensitive to the size of the input data. Note that since we add a lot of noise (for the length of 9000, only 1.4\% are not noise), the Rand Index in the right figure drops significantly. In the next section, we will inject a reasonable proportion of noise to analyze noise sensitivity.

From Table 1, we can find that there are some models that also have the same characteristics, namely linear complexity w.r.t dataset size and time series length, such as k-means, SPF and MiniRocket, but our model is overall more accurate than these methods and has superior performance on all evaluated time series data types.

\subsection{Analyzing noise sensitivity}
We use three different datasets, SmallKitchenAppliances, ECG200, and FiftyWords, from three different application domains to test the noise sensitivity of the model. These datasets are injected with six levels of random Gaussian noise (scales of 0.05, 0.1, 0.2, 0.3, 0.4, and 0.5). This setting ensures that most values in the time series are valid, unlike in the previous section, where most values are noise. We evaluate the performance of RandomNet against the second-best model, SPF, by running each model 10 times and calculating the average Rand Index.

As illustrated in Fig. ~\ref{fig:Noise}, while both models exhibit a strong resilience to noise, our model is slightly better than SPF. For the SmallKitchenAppliances dataset, the performance of RandomNet has little effect as the noise level increases. On the contrary, the performance of SPF decreases more obviously. In the ECG200 dataset, both models experience small fluctuations in performance at different noise levels, indicating similar effects on noise in this case. For the FiftyWords dataset, both models remain highly stable and show minimal performance differences despite the introduced noise.

Overall, these observations highlight RandomNet's competitive ability to handle noise, confirming its effectiveness and robustness in noisy scenarios.

\begin{figure}
    \centering
    \includegraphics[width=1\textwidth]{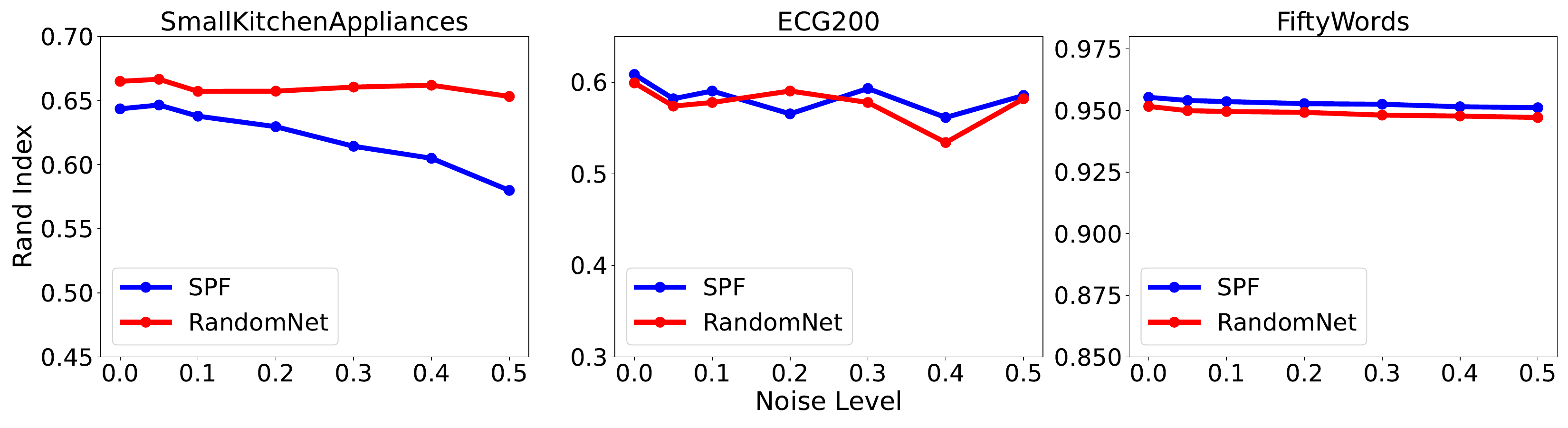}
    \caption{Sensitivity analysis of RandomNet and SPF across varying noise levels.}
    \label{fig:Noise}
\end{figure}

\subsection{Finding the optimal number of clusters}
In many real-world data mining scenarios, the true number of clusters ($k$) within the dataset is unknown, so whether the model has the ability to determine the optimal $k$ is crucial. The Elbow Method is a widely accepted heuristic used in determining the optimal $k$. It entails plotting the explained variation as a function of $k$ and picking the "elbow" of the curve as the optimal $k$ to use.

We apply the Elbow Method to the clustering performed by both k-means and RandomNet on the CBF dataset, which contains three classes. As shown in Fig. \ref{fig:elbow}, RandomNet can find an obvious "elbow" at $k=3$, whereas for k-means, it is hard to locate a clear "elbow". 

\begin{figure}[]
\centering
\includegraphics[width=0.75\textwidth]{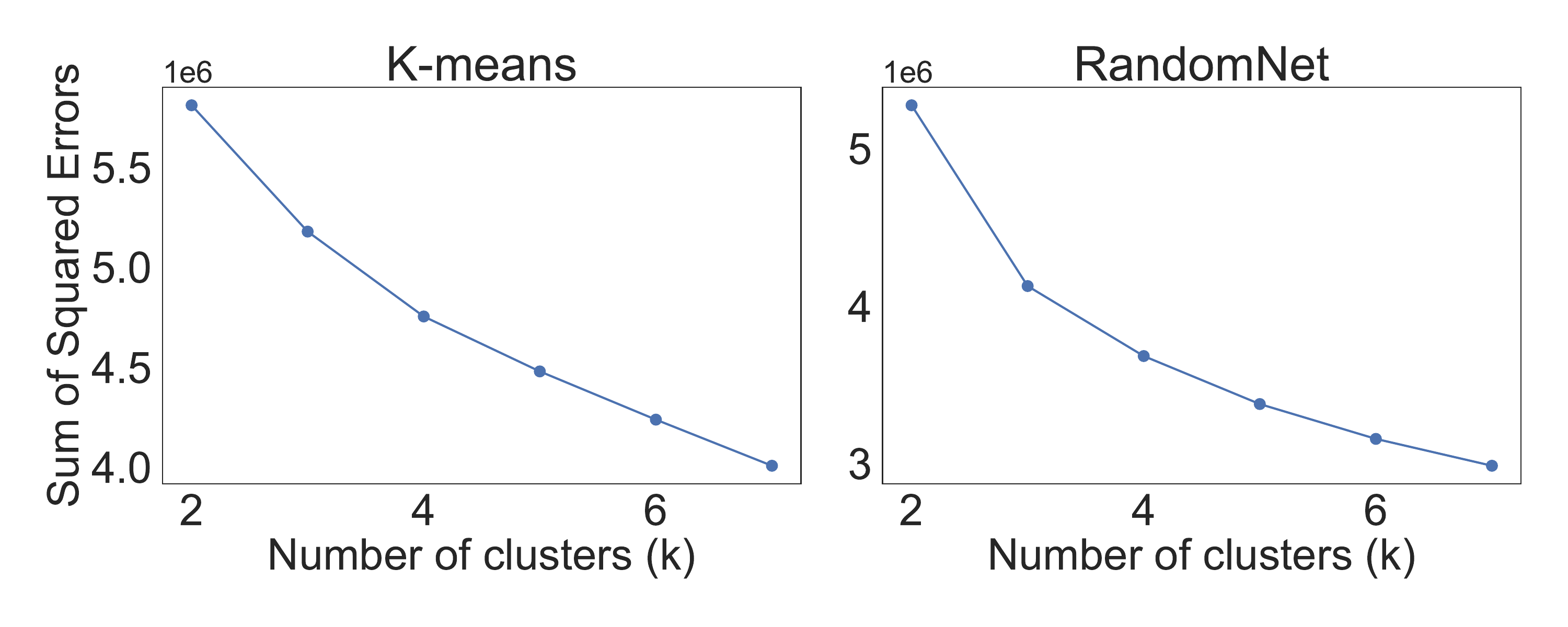}
\caption{Elbow Method test for k-means (left) and RandomNet (right) on CBF dataset.}
\label{fig:elbow}
\end{figure}

\section{Conclusion and Future Work}
\label{sec:conclusion}

In this paper, we introduces RandomNet, a novel method for time series clustering that utilizes deep neural networks with random parameters to extract diverse representations of the input time series for clustering. The data only passes through the network once, and no backpropagation is involved. The selection mechanism and ensemble in the proposed method cancel irrelevant representations out and strengthen relevant representations to provide reliable clustering. Extensive evaluations conducted across all 128 UCR datasets demonstrate competitive accuracy compared to state-of-the-art methods, as well as superior efficiency. Future research directions may involve integrating more complex or domain-specific network structures into our method. Additionally, incorporating some level of training into the framework could potentially improve performance. We will also try to explore the potential of applying our method to multivariate time series or other data types, such as image data. 


%\bibliographystyle{spphys} 
\bibliography{dmkd}% common bib file

\end{document}